\documentclass[twoside,11pt]{article}

\usepackage{blindtext}

\usepackage{jmlr2e}

\usepackage{bbm}
\usepackage{amsmath}
\usepackage{mathtools}
\usepackage{thmtools,thm-restate}
\usepackage{booktabs}
\usepackage{tikz}
\usetikzlibrary{shapes,arrows}
\tikzstyle{io} = [trapezium,trapezium left angle=70,trapezium right angle=-70,minimum height=0.6cm, draw, fill=red!20, 
text width=2.5em, text badly centered, node distance=3cm, inner sep=0pt]
\tikzstyle{decision} = [diamond, draw, fill=red!20, 
text width=2.em, text badly centered, node distance=4cm, inner sep=0pt]
\tikzstyle{block} = [rectangle, draw, fill=blue!20, 
text width=4.em, text centered, rounded corners, minimum height=2em]
\tikzstyle{line} = [draw, -latex']
\tikzstyle{cloud} = [draw, ellipse,fill=green!20, node distance=3cm,
minimum height=2em]
\tikzstyle{circle} = [draw, ellipse,fill=green! node distance=1cm,
minimum height=1.em]

\usepackage{enumitem}
\usepackage{xcolor}
\usepackage{url}

\usepackage{algorithm}
\usepackage{algpseudocode}
\makeatletter
\renewcommand*\env@matrix[1][*\c@MaxMatrixCols c]{%
  \hskip -\arraycolsep
  \let\@ifnextchar\new@ifnextchar
  \array{#1}}
\makeatother

\DeclareMathAlphabet\mathbfcal{OMS}{cmsy}{b}{n}

\newcommand{\mat}[1]{\mathbf{#1}}

\usepackage{lastpage}

\ShortHeadings{Asymptotically Fair Participation from an Optimal Control Perspective}{Zhuotong Chen, Qianxiao Li, Zheng Zhang}
\firstpageno{1}

\begin{document}

\title{Asymptotically Fair Participation in Machine Learning Models: an Optimal Control Perspective}

\author{\name Zhuotong Chen \email ztchen@ucsb.edu \\
      \addr Department of Electrical and Computer Engineering\\
      University of California at Santa Barbara, \\
      Santa Barbara, CA, USA
      \AND
      \name Qianxiao Li \email qianxiao@nus.edu.sg \\
      \addr Department of Mathematics, \\ National University of Singapore \\ Singapore, 119076
      \AND
      \name Zheng Zhang \email zhengzhang@ece.ucsb.edu \\
      \addr Department of Electrical and Computer Engineering\\
      University of California at Santa Barbara, \\
      Santa Barbara, CA, USA}

\editor{My editor}

\maketitle

\begin{abstract}
The performance of state-of-the-art machine learning models often deteriorates when testing on demographics that are under-represented in the training dataset.
This problem has predominantly been studied in a supervised learning setting where the data distribution is static.
However, real-world applications often involve distribution shifts caused by the deployed models.
For instance, the performance disparity against minority users can lead to a high customer churn rate,
thus the available data provided by active users are skewed due to the lack of minority users.
This feedback effect further exacerbates the disparity among different demographic groups in future steps.
To address this issue,
we propose asymptotically fair participation as a condition to maintain long-term model performance over all demographic groups.
In this work, we aim to address the problem of achieving asymptotically fair participation via optimal control formulation.
Moreover,
we design a surrogate retention system based on existing literature on evolutionary population dynamics to approximate the dynamics of distribution shift on active user counts,
from which the objective of achieving asymptotically fair participation is formulated as an optimal control problem and the control variables are considered as the model parameters.
We apply an efficient implementation of Pontryagin’s maximum principle to estimate the optimal control solution.
To evaluate the effectiveness of the proposed method,
we design a generic simulation environment that simulates the population dynamics of the feedback effect between user retention and model performance.
When we deploy the resulting models to the simulation environment,
the optimal control solution accounts for long-term planning and leads to superior performance compared with existing baseline methods.
\end{abstract}


\section{Introduction}
\label{sec: introduction}
In a dynamically changing environment, the data distribution can change over time yielding the phenomenon of concept drift.
Concept drift refers to changes in the conditional distribution of the target variable given the input features, while the distribution of the input may stay unchanged \citep{gama2014survey, schlimmer1986incremental, widmer1996learning}.
This paper delves into a nuanced aspect of concept drift in the realm of machine learning. 
Specifically, it examines situations where the efficacy of a machine learning model has a direct influence on the number of active users, while when the feature distribution of those users remains constant. 
To illustrate this, imagine a context wherein users are constantly communicating with their personal digital devices, such as virtual assistants, smart speakers, or even advanced wearables. 
These interactions typically involve input features like voice commands, gestures, or other forms of user inputs.
In such a scenario, the device's performance is not merely measured by a technical metric but directly influences the user experience. 
A virtual assistant that consistently understands and executes voice commands correctly will foster user trust and dependence. 
On the contrary, an assistant that often misinterprets commands or fails to execute tasks might frustrate users. 
Thus, positive predictions and successful task executions could bolster the user base, as satisfied users are more likely to continue using the device and even recommend it to others. 
Conversely, a series of negative outcomes, such as misunderstood commands or incorrect task executions, could deter users, causing a decline in user retention rates. 
The intricacy of this relationship becomes even more evident when analyzing diverse demographic groups \citep{harwell2018amazon}. 
For example, younger users might be more forgiving of occasional glitches and continue using the device, while older users, who are generally less tech-savvy, might get discouraged and abandon it altogether after a few negative experiences. 
Similarly, cultural nuances might make certain groups more patient or more demanding. 
A demographic used to high service standards might expect the virtual assistant to understand and process commands in dialects or regional accents. 
Failing to meet such expectations could result in significant user attrition for that particular demographic.

In a non-stationary environment,
the flow of participative users poses particular challenges, especially for minority demographic groups.
Such non-stationary environments can inadvertently lead to these groups experiencing a disproportionate share of system errors.
Consider a scenario wherein a digital system, such as voice recognition software, is continuously fine-tuned based on the majority's input. 
Minority users, who may possess distinct accents or linguistic patterns, find that their interactions result in frequent misunderstandings by the system. 
Over time, the consistent lack of system efficiency for these users creates a sentiment of alienation. 
These users may begin to question the utility of the system for them, given that their unique requirements seem persistently unmet.
As they experience these setbacks, there's a risk that these minority users will reduce their level of interaction with the system, or in extreme cases, stop using it altogether. 
This further deprives the system of valuable data inputs from these users. 
In essence, the system, which is already under-representing them, gets even less data from them, leading to an even sharper decline in the efficacy of its responses to this group.
This decline manifests in a detrimental feedback loop: As minority groups reduce their engagement due to high error rates, the data pool from these groups diminishes. 
The reduced data input further impedes the system's ability to improve its performance for these groups, leading to even higher error rates in the future. 
Consequently, the issue compounds, and the system's inefficiency for these groups becomes even more entrenched \citep{liu2018delayed}.
Figure \ref{fig: structure} illustrates this concept, starting with the initial feature distributions of two distinct demographic groups (shown in Figure \ref{fig: structure} (a)), where the left and right blobs represent the features of the majority and minority groups, respectively. 
The use of population risk minimization, a common objective for generating predictive models, results in undesired patterns as shown in Figure \ref{fig: structure} (b),
where the decision boundary of the predictive model at the terminal stage shows a clear bias towards the majority group, which leads to the diminishing of minority users.
The impact of this bias on the population densities of users interacting with the predictive model is further highlighted in Figure \ref{fig: structure} (d).

Several existing studies have attempted to address the aforementioned fairness problem in non-stationary environments. 
A predominant approach involves the implementation of conventional fairness algorithms crafted for static settings \citep{hashimoto2018fairness}.
Under this paradigm, the algorithm is applied iteratively at every discrete time step, with the aim to optimize the model parameters so that they meet a particular objective function's optimality criteria for that instance. 
While such an approach seems straightforward and can cater to immediate requirements, it often carries an inherent myopia. 
By being overly focused on achieving local optimality at each step, it may inadvertently miss out on discerning a more global solution that optimally serves the system throughout its entire simulation horizon.

\begin{figure}[t]
\centering
\begin{minipage}{.9\linewidth}
\centering
\includegraphics[width=\linewidth]{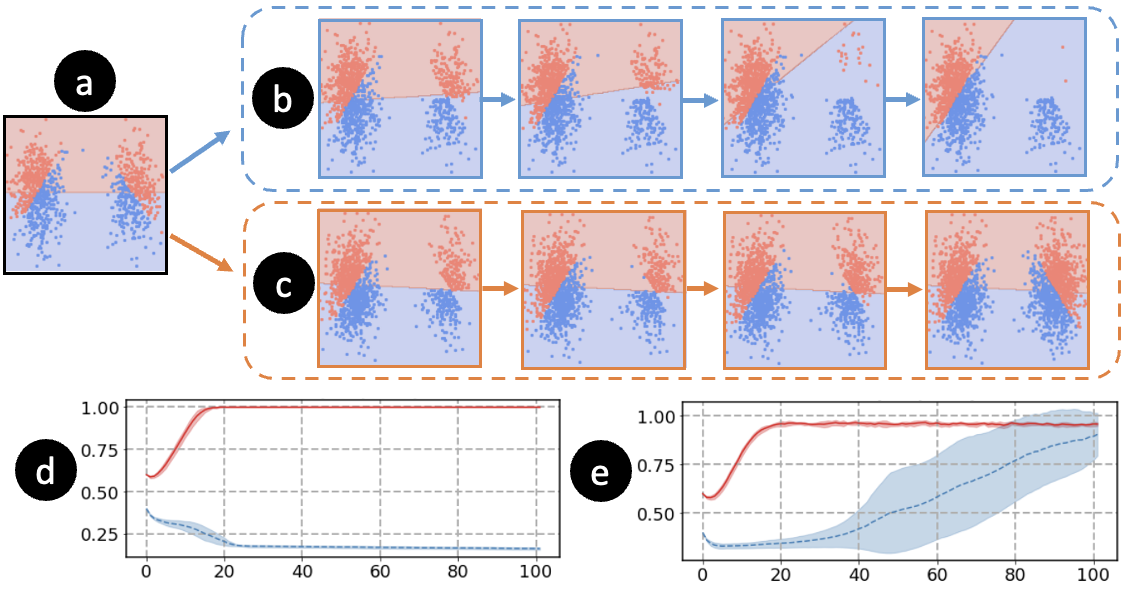}
\end{minipage}
\caption{
(a) illustrates the initial feature distributions for two distinct demographic groups. 
The variations in feature distribution among engaged users, resulting from population risk minimization and the proposed optimal control method, are presented in (b) and (c). 
(d) and (e), respectively, reveal the changes in the population densities of these users over time.
}
\label{fig: structure}
\end{figure}

In our study, 
we aim to construct machine learning models that engage users from all demographic groups, amid the described population dynamics.
This task can be construed as a trajectory planning problem \citep{bellman1952theory}, where the trajectory represents the evolution of user participation willingness, and our aim is to optimize this trajectory using judicious control design.
However, adopting such an approach mainly poses four challenges.
First and foremost, the concept of fairness in a non-stationary environment lacks a formal definition.
Existing fairness definitions focus on measuring performance disparities between different demographic groups at a single time point like equal opportunity \citep{hardt2016equality}, and demographic parity \citep{feldman2015certifying}.
Yet, these definitions, being focused on static supervised learning environments, fail to encapsulate the goals in a non-stationary environment.
Second,
A requirement for the optimal control methodology is a comprehensive understanding of the system's underlying dynamics. 
When it comes to user populations, these dynamics are often intricate, multi-faceted, and elusive, making them difficult to precisely model or predict.
Thirdly,
optimal control problems, by their nature, are computationally demanding,
which poses challenges to obtaining the optimal control solution.
Lastly, evaluation of the performance of an optimal control solution necessitates its deployment on population dynamics with real-world users, an often costly and unattainable requirement. 
We address these challenges with the following contributions:
\begin{enumerate}
    \item
    We introduce the concept of asymptotically fair participation to describe the maintenance of performance across all demographic groups over an extended period.
    \item
    We introduce a surrogate retention system, drawing inspiration from the rich body of work on evolutionary population dynamics,
    from which we formulate the objective of achieving asymptotically fair participation as an optimal control problem. 
    To address this complex problem, we employ an efficient implementation of Pontryagin's maximum principle (PMP), which allows us to obtain the control solution efficiently.
    \item 
    We prove that the proposed optimal control method produces models with monotonically non-decreasing performance for each update,
    this analysis motivates a novel form for the Hamiltonian in PMP.
    \item
    We design a simulator that simulates the non-stationary environment for the user's willingness to retain or churn from a deployed model.
    This simulator allows for testing the evolutionary fairness property of machine learning models in synthetic population dynamics.
    Through empirical evaluation,
    we underscore the benefits of incorporating the underlying dynamics into our model design.
    Our results consistently outperform existing baseline methods, thereby validating the superiority of our approach in terms of performance. 
\end{enumerate}

Figure \ref{fig: structure} (c) and (e) showcase the results of the proposed optimal control method. 
In these figures, the decision boundary of the predictive model exhibits a mild bias in favor of the minority demographic group. 
This approach successfully sustains the engagement of users from the minority group at every time step, thereby averting any bias in the predictive models towards the majority group in subsequent time steps. 

\section{Fairness in Non-Stationary Environment}
In this section,
we elaborate on the problem configuration to fair user participation in a non-stationary environment, where user retention and churn are conditioned on the model's performance on the data they provide (Section \ref{sec: problem description}). 
The dynamics of such an environment necessitate a condition for the machine learning models to fulfill, which we term asymptotically fair participation. 
This condition requires the models to sustain their performance across all demographic groups over an extended period (Section \ref{sec: definition of asymptotically fair participation}).

\subsection{Problem Setting for Fair Participation in a Non-Stationary Environment}
\label{sec: problem description}
In a non-stationary environment,
we consider a predictive model as a sequence of models,
denoted as $\{ \boldsymbol\theta_t \}_{t=0}^{T-1}$ where each model  $\boldsymbol\theta_t \in \mathbb{R}^m$ and $m$ is the number of parameters.
Within this environment, we focus on $K$ different demographic groups. 
Each of these groups consists of $N^i$ users where $i$ refers to the index of the demographic group,
and the total number of users is represented as $N = \sum_{i=1}^K N^i$.
These users include both participative and non-participative users of the predictive model,
and the $\rm n^{th}$ user is represented by a feature vector ($\mat{x}_n \in \mathbb{R}^d$), a label ($y_n \in \mathbb{R}$),
and a demographic membership ($z_n \in \mathbb{R}$).
The predicted output of each model is denoted as $\hat{y} = \boldsymbol\theta_t(\mat{x})$.

Now, we formulate the population dynamics of the user's willingness to participate as a Markov Decision Process (MDP) and denote it as \textbf{population retention system},
\begin{itemize}
    \item 
    \textbf{States}:
    The state is described by a binary vector $S_t \in \{0, 1\}^N$ that indicates whether each user is participative with respect to the predictive model at time step $t$.
    For instance,
    the $\rm n^{th}$ user is participative if $[S_t]_n = 1$ and non-participative if $[S_t]_n = 0$.
    \item
    \textbf{Actions}:
    Actions are the model outputs $\{\hat{y}_n\}_{n=1}^N = \{\boldsymbol\theta_t(\mat{x}_n)\}_{n=1}^N$, which are predictions derived from the feature vectors of all users. 
    Specifically, in the context of binary classification, an action takes the form of a binary vector with a length of $N$. 
    This vector categorizes each individual with either a positive or negative outcome.
    Moreover,
    only actions of participative users (where $[S_t]_n = 1$) are leveraged by the transition probability that is specified later.
    \item
    \textbf{Rewards}:
    At a time $t$, a reward $R(S_t)$ is measured from the current state $S_t$.
    Let us denote $\lambda_t^i$ as the population density of participative users from the $\rm i^{th}$ demographic group,
    \begin{equation*}
    \lambda_t^i = \frac{1} {N^i} \cdot \sum_{n=1}^N [S_t]_n \cdot \mathbbm{1}_{z_n = i},
    \end{equation*}
    where $\mathbbm{1}_{z_n = i}$ is a indicator function that returns $1$ if $z_n = i$, and $0$ otherwise.
    This density value is between $0$ and $1$.
    If $\lambda_t^i$ increases, it means more users from the $\rm i^{th}$ demographic group become participative in the predictive model.
    Conversely, a decrease suggests users from that group are leaving or becoming non-participative.
    We compute rewards by measuring the population densities, one example could be the sum of population densities $R(S_t) = \sum_{i=1}^K \lambda_t^i$.
    \item
    \textbf{Transition probability}:
    The transition probability characterizes the changes in the participative status of the user.
    The status of the $\rm n^{th}$ user at time step $t+1$ is conditioned on the $\rm n^{th}$ action $\boldsymbol\theta_t(\mat{x}_n)$.
    We assume that the probability of the user's participation is proportional to the model performance of each particular user \citep{riedl2013tweeting, huang2009searching}.
    A model with higher satisfaction (e.g. correct prediction) leads to a higher probability of user retention, meaning that $[S_{t+1}]_n$ is more likely to be $1$.
    Conversely,
    wrong prediction results in a higher possibility of user churn, meaning that 
    the likelihood of $[S_{t+1}]_n = 0$ is high.
    This is defined by the transition probability,
    \begin{align}
    \label{eq: population retention system}
    S_{t+1} \sim M^{\ast}(\cdot | S_t, & \{\boldsymbol\theta_t(\mat{x}_n)\}_{n=1}^N, \{y_n\}_{n=1}^N, \{z_n\}_{n=1}^N),
    \nonumber \\
    {\rm where}
    \;\;
    [S_{t+1}]_n \sim
    &
    \begin{cases} 
    {\rm Bernoulli(1)} \;\; & \text{if } \boldsymbol\theta_t(\mat{x}_n) = y_n, 
    [S_t]_n = 1, \\
    {\rm Bernoulli(\epsilon)} & \text{if } \boldsymbol\theta_t(\mat{x}_n) \neq y_n, [S_t]_n = 1, \\
    {\rm Bernoulli(\alpha)} & \text{if } [S_t]_n = 0,
    \end{cases} 
    \end{align}
    where $\epsilon$ and $\alpha$ are small positive numbers that introduce randomness in transitions.
    \item 
    \textbf{Initial states}:
    Given the initial population densities,
    the initial states are constructed by randomly sampling participative users from each demographic group.
\end{itemize}
The setup of the MDP slightly deviates from the general case, where the agent (here the model) can take any action in the action space.
In this framework, actions are produced by a predetermined, parameterized model, such as a neural network. 
Additionally, these actions can only be updated by data obtained from users who are currently participating.
The proposed MDP consists of two stages.
\paragraph{Model generation:}
During this stage,
our goal is to create a sequence of $T$ models, represented by $\{\boldsymbol\theta_t\}_{t=0}^{T-1}$,
from interacting with the population retention system in Eq.~\eqref{eq: population retention system}.
At every time step,
we use the feature vectors and labels of participative users, along with feedback on rewards, to improve model performance.
However, if a user stops participating in the following time step, we no longer have access to their data.
\paragraph{Model evaluation:}
To assess the performance of the generated models, we deploy them into the population retention system, initiating from a random starting point, and evaluating the reward based on the observed population densities $\lambda_0^i, \lambda_1^i,...,\lambda_T^i$.

\begin{figure}[t]
\centering
\begin{minipage}{.9\linewidth}
\centering
\includegraphics[width=\linewidth]{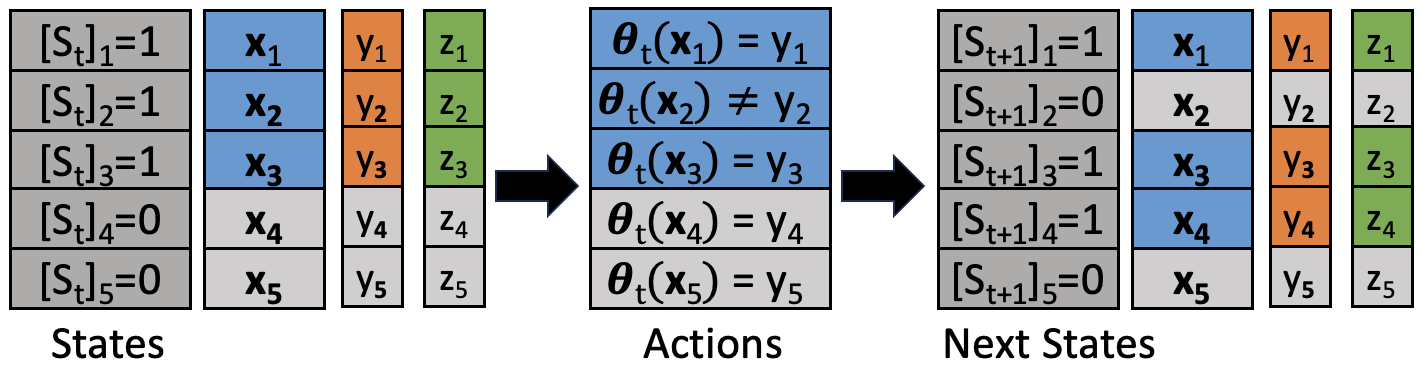}
\end{minipage}
\caption{
This figure represents the MDP transition as outlined in Eq.~\eqref{eq: population retention system}. 
At time step $t$, the first, second, and third users are actively participating. However, due to an incorrect prediction made by the model $\boldsymbol\theta_t$, for the second user, the user becomes non-participative in the following time step. 
On the other hand, the fourth user, who is non-participative at $t$, 
has a slight chance of becoming participative in the next time step.
}
\label{fig: mdp}
\end{figure}

Figure \ref{fig: mdp} details the proposed MDP.
At time step $t$,
the state $S_t$ reveals that the first three users are actively participating.
However,
due to an erroneous prediction by the predictive model $\boldsymbol\theta_t(\mat{x}_2) \neq y_2$,
the second user discontinues participation in the next time step, leading to $[S_{t+1}]_2  = 0$.
Regarding users who are not participating at $t$, there is a small chance they will engage with the model in the following time step, exemplified by the behavior of the fourth user.

\subsection{The Definition for Asymptotically Fair Participation}
\label{sec: definition of asymptotically fair participation}
The population retention system, defined in Eq.~\eqref{eq: population retention system}, 
simulates variation in the user's activity in participating as conditioned on the model performance of each particular individual.
In this context,
a predictive model with minimum population risk across the entire population encourages increased user participation within the system.
Consequently, active users are more inclined to provide additional training data, which is invaluable for refining the model performance.
With these newly contributed data, the predictive model undergoes further refinement. 
This iterative process progressively improves the model's predictive accuracy, thereby reducing population risk even further in each succeeding time step. 
Given this continuous enhancement and the relationship between user participation and model performance, there exists a positive feedback loop.
With sufficient iterations and time, this loop leads to a scenario where the population risks associated with all demographic groups tend towards $0$. 
This occurs as the number of total users denoted as $N^i$ grows significantly larger. 
Simultaneously, the population densities of all demographic groups approach $1$,
indicating high levels of engagement and participation across all demographics.
This motivates us to define asymptotically fair participation as follows:
\begin{definition}{\bf (Asymptotically fair participation)}
\label{def:fairness definition}
A sequence of models satisfies asymptotically fair participation if the dynamics it drives satisfy the following condition:
\begin{equation*}
    \lambda_t^i \rightarrow 1 
    \;\;
    {\rm almost \; surely}, \;\; {\rm as} \; t \rightarrow \infty \;\; \forall i \in [1, 2, ..., K], \;\; {\rm s.t. ~Eq.~} \eqref{eq: population retention system}.
\end{equation*}
\end{definition}
When the population densities of participative users from each demographic group converge towards $1$, it indicates that the underlying predictive model consistently performs fairly across all these groups.
Moreover,
the satisfaction of asymptotically fair participation by a sequence of models is implicitly linked to the initial population densities.
In scenarios where all demographic groups initially have high population densities, the likelihood of achieving asymptotically fair participation increases. 
Conversely, scenarios with highly imbalanced representations of demographic groups pose significant challenges in meeting this condition.
Therefore, the initial representation of demographic groups plays a critical role in the implementation of models to the condition of asymptotically fair participation.

The concept of asymptotically fair participation provides a more precise interpretation of fairness in a non-stationary environment.
Prior research has considered disparity amplification \citep{hashimoto2018fairness} to assess the representation disparity across all demographic groups at each individual time step.
However, the definition of asymptotically fair participation differs from this approach as it emphasizes long-term behavior.
To illustrate, consider an extreme scenario where the population densities of all demographic groups concurrently decay to zero.
Although this is an undesirable situation, it would nonetheless satisfy the condition of disparity amplification, yet not meet the criterion of asymptotically fair participation. 
Thus, the distinction underscores the importance of considering long-term behavior in fairness definitions, a perspective that asymptotically fair participation uniquely encapsulates.

\section{An Optimal Control Solution for Asymptotically Fair Participation}
According to Definition \ref{def:fairness definition},
our goal is to maximize the population densities across all demographic groups.
Due to the inaccessibility of the underlying dynamics of the population retention system, as defined by Eq.~\eqref{eq: population retention system},
our initial step involves the construction of a surrogate system to estimate these dynamics.
Subsequently, we formulate the condition of asymptotically fair participation as an optimal control problem and provide an efficient solver based on Pontryagin's maximum principle (PMP) \citep{pontryagin1987mathematical}.

\subsection{Surrogate Retention System for the Evolutionary Population Dynamics}
\label{sec:surrogate retention system}
Our design of the surrogate retention system is rooted in the existing body of literature on evolutionary population dynamics \citep{cushing2019difference}.
This system features a low-dimensional state representation, which not only provides a meaningful connection to evolutionary dynamics but also offers practical advantages in terms of computational efficiency.

\paragraph{Evolutionary population dynamics describes the dynamics of user participation.}
Difference equations typically describe discrete-time dynamics such that the temporal variations in vital rates are attributable to dependencies on population density. 
An individual's behavior and activities, such as reproduction and survival, can undergo fluctuations, leading to the evolutionary dynamics of population density.
Explicit temporal dependencies can be modeled by optimizing the coefficients of a difference equation over time \citep{vincent2005evolutionary}. 
To account for such evolutionary mechanisms, a difference equation population model can be developed \citep{cushing2019difference}.
In a simplified scenario, the growth and decay of the population are attributed solely to births and deaths, respectively.
Individuals present at time $t+1$ either emerged during the time interval or were present at time $t$ and survived the time unit.
To model these dynamics, we denote $\boldsymbol\lambda_t=[\lambda_t^1, \lambda_t^2, ..., \lambda_t^K]^T$ as a $K$-dimensional vector.
Subsequently, a $K$-dimensional discrete dynamic system describing the simplified evolutionary population dynamics can be constructed as follows:

\begin{equation}
\label{eq: surrogate retention system}
\boldsymbol\lambda_{t+1} = M(\boldsymbol\lambda_t, \boldsymbol\theta_t)
= 
\begin{bmatrix}
\beta(\kappa^1(\lambda_t^1, \boldsymbol\theta_t))
\\
\beta(\kappa^2(\lambda_t^2, \boldsymbol\theta_t))
\\
\vdots \\
\beta(\kappa^K(\lambda_t^K, \boldsymbol\theta_t))
\\
\end{bmatrix}
\odot
(\mat{1} - \boldsymbol\lambda_t)
+
\begin{bmatrix}
\sigma(\kappa^1(\lambda_t^1, \boldsymbol\theta_t))
\\
\sigma(\kappa^2(\lambda_t^2, \boldsymbol\theta_t))
\\
\vdots \\
\sigma(\kappa^K(\lambda_t^K, \boldsymbol\theta_t))
\\
\end{bmatrix}
\odot
\boldsymbol\lambda_t,
\end{equation}
where $\odot$ denotes the element-wise product between two vectors.
The function $\kappa^i(\cdot)$ calculates a value indicative of the ${\rm i^{th}}$ population's reaction to external controls $\boldsymbol\theta_t$ (such as medical interventions or the distribution of resources). 
The functions $\beta(\cdot)$ and $\sigma(\cdot)$ are used to determine the proportions of births and the surviving population over a given time period, respectively. 
The domain for both the birth rate and survival rate functions is confined to the interval $[0, 1]$, which establishes a range for population densities.

In cases where user retention or churn rates impact population dynamics, the function $\kappa^i(\cdot)$ is employed to measure the model's performance on the currently active population.
The birth rate $\beta(\cdot)$ and the survival rate $\sigma(\cdot)$ illustrate the proportions of incoming and retained users at each respective time step.
Furthermore, when model performance is evaluated through a reward (or conversely, a loss) function, we hypothesize that the birth and survival rates act proportionally (or inversely) to the model performance $\kappa^i(\cdot)$. 
This assumption ensures that an improvement in model performance leads to an increase in population density. 
We refer to this discrete dynamic system as the \textbf{surrogate retention system}.

The surrogate retention system leverages a statistically aggregated metric, specifically population density, to represent the ratio of active users within each demographic group. 
This serves to condense the intricate state space of the population retention system as detailed in Eq.~\eqref{eq: population retention system}. 
Central to our approach is the hypothesis that this simplified state representation can effectively encapsulate the nuances of the population dynamics. 
As a result, we achieve a low-dimensional state representation where the defining state elements are the population densities across various demographic groups.

\begin{remark}
Numerous system estimations have been put forward to investigate the dynamic relationship between user engagement and data-driven services. 
For example, the algorithm outlined by \citep{hashimoto2018fairness} contemplates a straightforward system to account for the user count based on model performance. 
Another system, described by \citep{dean2022multi}, looks into the endogenous shift in distributions, focusing on how populations distribute themselves among services, and how these services, in turn, select predictors derived from the observed userbase.
This is distinct from our proposed surrogate retention system, which is rooted in the evolutionary population dynamics. 
Moreover, our system features optimizable parameters, aiding in approximating the population retention mechanisms. 
\end{remark}

\paragraph{Evaluation of model performance through distributionally robust optimization.}
The surrogate retention system, as characterized by Eq.~\eqref{eq: surrogate retention system},
yields a low-dimensional state representation, comprising solely of the population densities across all demographic groups.
However, its simulation necessitates the selection of data provided by active users based on the population densities.
This can be done in many ways.
For instance,
one can randomly sample participating users from the $\rm i^{th}$ demographic group at time step $t$,
or sampling proportional to the model performances of users.
The sampling-based data generation leads to a stochastic dynamic system,
which creates challenges in solving the optimal control problem.
In this work,
we consider the formulation of distributionally robust optimization (DRO), which facilitates a deterministic generation process of $\lambda_t^i$ proportion of users who received optimal model performance.

To begin with,
let $d_{\mathcal{X}^2}(\mathcal{P}||\mathcal{Q}) = \int (\frac{d \mathcal{P}} {d\mathcal{Q}} - 1)^2 d\mathcal{Q})$ denote the $\mathcal{X}^2$-divergence between two probability distributions $\mathcal{P}$ and $\mathcal{Q}$,
$\mathcal{B}(\mathcal{P}, r) = \{\mathcal{Q}: d_{\mathcal{X}^2}(\mathcal{P}||\mathcal{Q}) \leq r\}$ denote the chi-squared ball around a probability distribution $\mathcal{P}$ of radius $r$.
Let $\mathcal{P}^i$ be the feature distribution of users from the $\rm i^{th}$ demographic group,
we consider the performance measure $\kappa^i(\cdot)$ as the worst-case distributional loss over all r-radius balls around $\mathcal{P}^i$ defined as follows,
\begin{equation}
    \label{eq: worst-case distribution loss}
    \kappa^i(\lambda_t^i, \boldsymbol\theta_t) = \sup_{\mathcal{Q} \in \mathcal{B}(\mathcal{P}^i, r_t^i)} \mathbb{E}_{(\mat{x}, y) \sim \mathcal{Q}} \Phi(\boldsymbol\theta_t, \mat{x}, y), 
    \;\; 
    r_t^i = (1 / \lambda_t^i - 1)^2.
\end{equation}
Clearly,
as the number density $\lambda_t^i$ approaches $1$,
$r_t^i$ decays to $0$, and $\kappa^i(\lambda_t^i, \boldsymbol\theta_t)$ is equivalent to population risk.
For small $\lambda_t^i$,
the radius $r_t^i \rightarrow \infty$ and this leads to a large loss value.
In general,
computing the worst-case distributional loss over a set of distributions is a challenging task.
Fortunately,
the maximization problem in Eq.~\eqref{eq: worst-case distribution loss} can be reformulated into its dual form \citep{duchi2019distributionally}.
More specifically,
if $\Phi(\cdot)$ is upper semi-continuous for any $\boldsymbol\theta$,
then for $r_t^i \geq 0$ and any $\boldsymbol\theta$,
the following holds true:
\begin{align}
    \label{eq: dual of worst-case distribution loss}
    \sup_{\mathcal{Q} \in \mathcal{B}(\mathcal{P}^i, r_t^i)} \mathbb{E}_{(\mat{x}, y) \sim \mathcal{Q}} & \Phi(\boldsymbol\theta_t, \mat{x}, y)
    =
    \inf_{\eta \in \mathbb{R}} \Big(C(\lambda_t^i) \cdot \big( \mathbb{E}_{\mathcal{P}^i} \big[[\Phi(\boldsymbol\theta_t, \mat{x}, y) - \eta]_{+}^2 \big] \big)^{\frac{1} {2}} + \eta \Big), \nonumber\\
    & {\rm where} \;\; C(\lambda_t^i) = (2 (1 / \lambda_t^i - 1)^2 + 1)^{\frac{1}{2}},
\end{align}
where $[x]_+ = x$ if $x \geq 0$ and $0$ otherwise.
This dual form provides an intuitive interpretation of the DRO loss.
At each time step $t$,
given $\boldsymbol\theta_t$ and $\lambda_t^i$,
the DRO loss is computed by averaging the sample losses that are higher than the optimal $\eta^{\ast}(\lambda_t^i, \boldsymbol\theta_t)$,
where $\eta^{\ast}(\lambda_t^i, \boldsymbol\theta_t)$ attains the infimum.
Figure \ref{fig: dro} presents the features of participative users derived from the DRO formulation. 
In this illustration, the features of these users are predominantly clustered around the decision boundary, which corresponds to the maximum loss.

\begin{figure}[t]
\centering
\begin{minipage}{.4\linewidth}
\centering
\includegraphics[width=\linewidth]{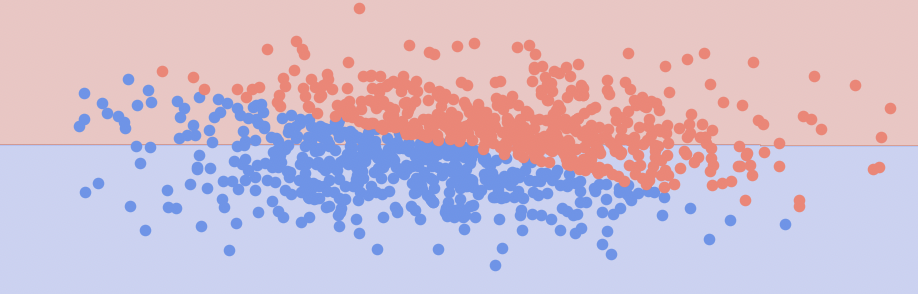}
(a): Population features
\end{minipage}
\begin{minipage}{.4\linewidth}
\centering
\includegraphics[width=\linewidth]{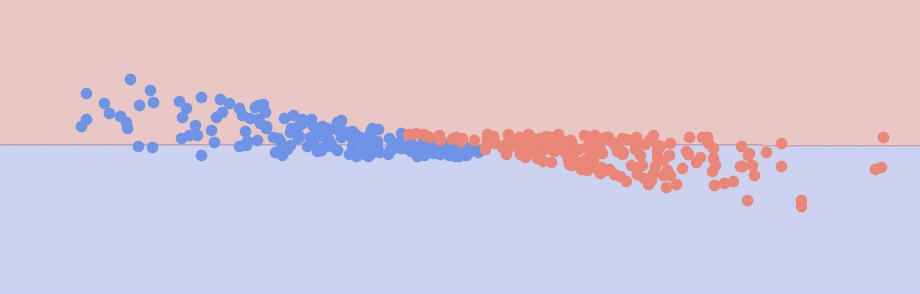}
(b): DRO features
\end{minipage}
\caption{
(a) and (b) plot the features of the entire population and participative user features resulting from DRO, respectively.
}
\label{fig: dro}
\end{figure}

The subsequent Proposition establishes that the DRO formulation offers a worst-case guarantee for the trajectory of population density that emerges from the surrogate retention system outlined in Eq.~\eqref{eq: surrogate retention system}.
The derivation is presented in Appendix \ref{sec: proof for the worst-case guarantee}.

\begin{restatable}{proposition}{WorstCaseGuarantee}
\label{theorem: worst-case guarantee}
Consider $\lambda_0^i$, $\lambda_1^i$,...,$\lambda_T^i$ as population densities derived from the surrogate retention system utilizing the DRO formulation, and $\hat{\lambda}_0^i$, $\hat{\lambda}_1^i$,...,$\hat{\lambda}_T^i$ as the sequence from the same system when population risk is applied.
Then,
\begin{equation*}
\lambda_t^i \leq \hat{\lambda}_t^i, 
\;\;
\forall t \in [0, T],
\;\;
i \in [1, K].
\end{equation*}
\end{restatable}

\subsection{Optimal Control Formulation for Asymptotically Fair Participation}
\label{sec:optimal control formualtion}
The definition of asymptotically fair participation, as outlined in Definition \ref{def:fairness definition}, requires that over an infinite period, the population densities of each demographic group approach and stabilize at $1$. 
The subsequent Proposition confirms that an equilibrium state signified by $\boldsymbol\lambda_t = \mat{1}$ is stable under a certain condition. 
This suggests that upon reaching the equilibrium state of $\mat{1}$, the population densities will remain in this state.
\begin{restatable}{proposition}{StableEquilibrium}
\label{proposition: stability of an fair equilibrium state}
In the surrogate retention system as described by Eq.~\eqref{eq: surrogate retention system}, a equilibrium state with $\boldsymbol\lambda_t = \mat{1}$ is stable if the following condition holds,
\begin{equation*}
\max_{i \in [1,2,...,K]}
\frac{\partial \sigma} {\kappa^i(\lambda_t^i, \boldsymbol\theta_t)}
\cdot
\frac{\partial \eta^{\ast}} {\partial \lambda_t^i}
\cdot
\Big(
1 -
\frac{\mathbb{E}_{\mathcal{P}^i} \big[\Phi(\boldsymbol\theta_t, \mat{x}, y)\big] }
{\sqrt{
\mathbb{E}_{\mathcal{P}^i} \big[\Phi(\boldsymbol\theta_t, \mat{x}, y)^2 \big]
}}
\Big)
< 1,
\end{equation*}
where $\eta^{\ast}$ is the optimal $\eta$ that achieves the infimum of the DRO dual expression.
\end{restatable}
The detailed proof is derived in Appendix \ref{sec: proof for the stability of the equilibrium state}.
The Proposition indicates that the stability of the population density state is dependent on the variance of $\Phi(\boldsymbol\theta_t, \mat{x}, y)$ across the entire population. This is logical because when each user experiences comparable losses, the DRO formulation becomes more consistent with the fluctuations in the population density $\lambda_t^i$.

In a practical scenario where there's a finite time frame, 
the condition of achieving asymptotically fair participation is equivalent to reaching a terminal state with $\boldsymbol\lambda_t = \mat{1}$ due to the stability of the equilibrium state.
To achieve this, it's beneficial to view model parameters across every single time step as time-dependent control variables. 
By doing so, we can formulate this into a trajectory optimization problem. 
In essence, this optimization challenge revolves around identifying the most effective sequence of controls in order to optimize the terminal state. 
This process is subject to the surrogate retention system, as defined in Eq.\eqref{eq: surrogate retention system}. 
To delve deeper into specifics, 
we use the symbol $\Psi(\boldsymbol\lambda_T)$ to represent a certain measurement of the terminal state, denoted by $\boldsymbol\lambda_T$ (e.g. $\Psi(\boldsymbol\lambda_T)$ is equivalent to the reward function $R(S_t)$, in which the reward function acts on the state, and $\Psi(\boldsymbol\lambda_T)$ acts on population densities that are computed from the state $S_t$). 
As an example, $\Psi(\boldsymbol\lambda_T) = \sum_{i=1}^K \lambda_T^i$ can be the sum of population densities at the terminal step.
Alternatively, one might interpret $\Psi(\boldsymbol\lambda_T)$ as the negative binary cross-entropy loss when comparing $\boldsymbol\lambda_T$ with a vector of ones, denoted by $\boldsymbol{1}$. 
The rationale behind this is that maximizing this particular measurement aligns with our original goal: maximizing the population densities of all groups at the final time step.
Then the objective of achieving asymptotically fair participation can be formulated as follows:
\begin{equation}
\label{eq:progressive fairness objective}
\max\limits_{ \{ \boldsymbol\theta_t \}_{t=0}^{T-1} } \Psi(\boldsymbol\lambda_T)
\;\;
{\rm s.t.}
\boldsymbol\lambda_{t+1} = M(\boldsymbol\lambda_t, \boldsymbol\theta_t),
\;\; {\rm given} \; \boldsymbol\lambda_0,
\end{equation}
where $M(\cdot)$ is the surrogate retention system defined in Eq.~\eqref{eq: surrogate retention system}.
This is a special case of a class of general optimal control problems for discrete dynamical systems,
in which we consider the control variables as the model parameters at all time steps.
From this optimal control perspective,
asymptotically fair participation can be achieved by solving for a set of controls such that Definition \ref{def:fairness definition} is satisfied.

We describe a general solver for the optimal control problem in Eq.~\eqref{eq:progressive fairness objective} based on PMP.
PMP \citep{pontryagin1987mathematical} consists of two difference equations and a maximization condition. 
Instead of computing the state-dependent closed-loop control function, the PMP solves for a set of fixed control parameters for every initial state.
To begin with, 
we define the Hamiltonian as

\begin{equation}
\label{eq: hamiltonian}
    H(t, \boldsymbol\lambda_t, \mat{p}_{t+1}, \boldsymbol\theta_t) \vcentcolon = \mat{p}_{t+1}^{T} \cdot M(\boldsymbol\lambda_t, \boldsymbol\theta_t) - {\cal L}(\boldsymbol\theta_t, \boldsymbol\lambda_t),
\end{equation}
where ${\cal L}(\boldsymbol\theta_t, \lambda_t^i)$ is a running loss at time $t$,
which can be defined as the regularization term of model parameters (We will discuss this running loss term in Section \ref{sec: theoretical analysis}).
The PMP consists of a two-point boundary value problem,

\begin{align}
    & \boldsymbol\lambda_{t+1}^{\ast} = \nabla_p H(t, \boldsymbol\lambda_t^{i, \ast}, \mat{p}_t^{\ast}, \boldsymbol\theta_t^{\ast}),
    \;\;\;\;
    \boldsymbol\lambda_0 \;\; {\rm given},
    \label{eq:pmp forward} \\
    & \mat{p}_t^{\ast} = \nabla_{\lambda} H(t, \boldsymbol\lambda_t^{i, \ast}, \mat{p}_{t+1}^{\ast}, \boldsymbol\theta_t^{\ast}),
    \;\;\;\;
    \mat{p}_T = \frac{\partial \Psi(\boldsymbol\lambda_T)} {\partial \boldsymbol\lambda_T},
    \label{eq:pmp backward}
\end{align}
plus a maximum condition of the Hamiltonian.

\begin{equation}
    H(t, \boldsymbol\lambda_t^{i, \ast}, \mat{p}_t^{\ast}, \boldsymbol\theta_t^{\ast}) \geq H(t, \boldsymbol\lambda_t^{i, \ast}, \mat{p}_t^{\ast}, \boldsymbol\theta_t), \;\; \forall \; \boldsymbol\theta_t \;{\rm and}\; t.
    \label{eq:pmp maximization}
\end{equation}

We consider the method of successive approximation \citep{kirk1970optimal, li2018maximum, JMLR:v23:22-0529} to solve for the control solution.
Given a initial condition $\boldsymbol\lambda_0$,
notice in Eq.~\eqref{eq:pmp forward},
\begin{equation*}
\boldsymbol\lambda_{t+1}^{\ast} = \nabla_p H(t, \boldsymbol\lambda_t^{i, \ast}, \mat{p}_t^{\ast}, \boldsymbol\theta_t^{\ast})
=
M(\boldsymbol\lambda_t^{\ast}, \boldsymbol\theta_t),
\end{equation*}
which is equivalent to the forward propagation of the surrogate retention system.
Once we reach the terminal state $\boldsymbol\lambda_T$,
the adjoint system defined in Eq.~\eqref{eq:pmp backward} is a difference equation that propagates the derivative of the terminal loss w.r.t. state $\boldsymbol\lambda_t$ at every time step $t$.
The adjoint state at each time step can be represented as
\begin{align*}
\frac{\partial \Psi(\boldsymbol\lambda_T)} {\partial \boldsymbol\lambda_t}
& =
\frac{\partial \Psi(\boldsymbol\lambda_T)} {\partial \boldsymbol\lambda_T}
\cdot
\frac{\partial \boldsymbol\lambda_T} {\partial \boldsymbol\lambda_{T-1}}
\cdots
\frac{\partial \boldsymbol\lambda_{t+2}} {\partial \boldsymbol\lambda_{t+1}}
\cdot
\frac{\partial \boldsymbol\lambda_{t+1}} {\partial \boldsymbol\lambda_t},
\\
& = 
\frac{\partial \Psi(\boldsymbol\lambda_T)} {\partial \boldsymbol\lambda_{t+1}} ^T \cdot \frac{\partial \boldsymbol\lambda_{t+1}} {\partial \boldsymbol\lambda_t},
\\
& =
\mat{p}_{t+1}^T \cdot \frac{\partial M(\boldsymbol\lambda_t, \boldsymbol\theta_t)} {\partial \boldsymbol\lambda_t},
\end{align*}
which resembles the adjoint system defined in Eq.~\eqref{eq:pmp backward}.
Once we obtain the state $\boldsymbol\lambda_t$ and adjoint state $\mat{p}_t$,
the Hamiltonian can be optimized with respect to model parameters $\boldsymbol\theta_t$,
\begin{equation*}
\boldsymbol\theta_t^{\ast}
=
\arg\max_{\boldsymbol\theta} H(t, \boldsymbol\lambda_t, \mat{p}_{t+1}, \boldsymbol\theta).
\end{equation*}
this can be solved via any optimization method (e.g. gradient ascent).

Instead of iterating through all three Hamiltonian dynamics for a single update on the control solutions,
we can consider optimizing the $\rm t^{th}$ Hamiltonian locally for all $t \in [0,\cdots,T-1]$ with the current state $\boldsymbol\lambda_t$ and adjoint state $\mat{p}_{t+1}$.
This allows the control solution $\boldsymbol\theta_t$ to be updated multiple times within one complete iteration.
Once a locally optimal control $\boldsymbol\theta_t^{\ast}$ is achieved by maximizing $H(t, \boldsymbol\lambda_t^{i, \ast}, \mat{p}_{t+1}^{\ast}, \boldsymbol\theta_t^{\ast})$,
the adjoint state $\mat{p}_{t + 1}$ is backpropagated to $\mat{p}_t$ via the adjoint dynamic in Eq.~\eqref{eq:pmp backward} followed by maximizing $H(t-1, \boldsymbol\lambda_{t-1}^{i, \ast}, \mat{p}_t^{\ast}, \boldsymbol\theta_{t-1}^{\ast})$. 
In this configuration, executing the Hamiltonian dynamics $n$ times can be decomposed into $maxItr$ complete iterations and $InnerItr$ local updates.
Alg.~\ref{alg:control solver implementation} presents the method of successive approximation that solves the PMP iteratively.

\begin{algorithm}[tb]
   \caption{Method of Successive Approximation.}
   \label{alg:control solver implementation}
\begin{algorithmic}
    \State {\bfseries Input:} $\boldsymbol\lambda_0$,
    learning rate ${\rm lr}$,
    ${\rm maxItr}$, 
    ${\rm InnerItr}$
    \State {\bfseries Output:} models $\{ \boldsymbol\theta_t \}_{t=0}^T$
    \For{$m = 1$ to ${\rm maxItr}$}
    \For{$t = 0$ to T-1}
    \State $\boldsymbol\lambda_{t+1}^{\ast} = \nabla_p H(t, \boldsymbol\lambda_t^{i, \ast}, \mat{p}_t^{\ast}, \boldsymbol\theta_t^{\ast})$,
    {\color{blue} // Forward propagation of number densities (Eq.~\eqref{eq:pmp forward}).}
    \EndFor
    \State $\mat{p}_T^{\ast} = \frac{\partial \Psi(\boldsymbol\lambda_T)} {\partial \boldsymbol\lambda_T}$, {\color{blue} // Set terminal condition.}
    \For{$t=T-1$ to $0$}
    \For{$\tau = 0$ to ${\rm InnerItr}$}
    \State
    $H(t, \boldsymbol\lambda_t, \mat{p}_{t+1}, \boldsymbol\theta_t) = \mat{p}_{t+1}^{T} \cdot M(\boldsymbol\lambda_t, \boldsymbol\theta_t)$,
    {\color{blue} // Compute $H$ with $\mat{p}_{t+1}$ and $\boldsymbol\lambda_t$.}
    \State
    $\boldsymbol\theta_{t}^{\rm new} =
    \arg \max \limits_{\boldsymbol\theta} H(t, \boldsymbol\lambda_t, \mat{p}_{t+1}, \boldsymbol\theta_t)$,
    {\color{blue} // Maximize Hamiltonian (Eq.~\eqref{eq:pmp maximization}).}
    \EndFor
    \State
    $\mat{p}_t^{\ast} = \nabla_{\boldsymbol\lambda} H(t, \boldsymbol\lambda_t^{i, \ast}, \mat{p}_{t+1}^{\ast}, \boldsymbol\theta_t^{\ast})$,
    {\color{blue} // Backward propagation of adjoint state (Eq.~\eqref{eq:pmp backward}).}
    \EndFor
    \EndFor
\end{algorithmic}
\end{algorithm}

\subsection{Theoretical Analysis}
\label{sec: theoretical analysis}
In this section, we formulate an objective function designed to ensure that population densities do not decrease with each update of the model parameters. 
Through our theoretical analysis, we introduce a new form of the Hamiltonian, which leads to better convergence.

We consider an infinity horizon discounted reward setting,
where the reward function $R(S_t)$ is defined as a measurement of the population densities at time step $t$, as in Section \ref{sec: problem description} (e.g. the sum of population densities).
Moreover,
we use $M^{\ast}$ to indicate the population retention system defined in Eq.~\eqref{eq: population retention system} and $M$, $M'$ as its estimations.
We denote $V^{\boldsymbol\theta, M}$ as the value function of predictive model $\boldsymbol\theta = \{\boldsymbol\theta_t\}_{t=0}^{T-1}$ and dynamical system $M$,
\begin{equation*}
V^{\boldsymbol\theta, M}(s)
=
\mathbb{E}_{S_{t+1} \sim M(S_{t+1} | S_t, \boldsymbol\theta(S_t))}
\Big[
\sum_{t=0}^{\infty} \gamma^t R(S_t) | S_0 = s
\Big],
\end{equation*}
where the models $\boldsymbol\theta$ generate deterministic predictions,
$\gamma$ is a discounting factor.

\begin{restatable}{theorem}{maintheoremNonDecreaseValueFunction}
\label{theorem: nondecrease value function}
Let the value function satisfy L-Lipschitz continuity with a Lipschitz constant $L$.
Suppose $M^{\ast}$, representing a population retention system, is an element of the set $\mathcal{M}$, which denotes the space of estimated systems under consideration. 
When the optimality of the following objective is achieved,
\begin{gather*}
\boldsymbol\theta^{\rm new}, M^{\rm new}
=
\arg \max \limits_{\boldsymbol\theta, M}
V^{\boldsymbol\theta, M}
-
\Big(
\gamma \cdot L \cdot
\mathbb{E}_{\substack{S \sim \rho^{\boldsymbol\theta^{\rm old}, M^{\ast}}}}
\Big[
\lVert M(S, \boldsymbol\theta(S)) - M^{\ast}(S, \boldsymbol\theta(S)) \rVert
\Big]
+
\frac{2 B \gamma \kappa} {1 - \gamma}
\Big)
\\
{\rm s.t.}
\;\;
KL(\boldsymbol\theta^{\rm old}(S), \boldsymbol\theta(S))^{\frac{1} {2}} \leq \kappa,
\end{gather*}
where $\rho^{\boldsymbol\theta^{\rm old}, M^{\ast}}$ represents the stationary state distribution of the population retention system $M^{\ast}$ and model $\boldsymbol\theta^{\rm old}$,
$\lVert S \rVert \leq B$.
Then we have a non-decreasing value function of the population retention system from the resulting models,
\begin{equation*}
V^{\boldsymbol\theta^{\rm old}, M^{\ast}} \leq V^{\boldsymbol\theta^{\rm new}, M^{\ast}}.
\end{equation*}
\end{restatable}

The detailed proof is provided in Appendix \ref{sec: derivation for nondecreasing value function}.
Theorem \ref{theorem: nondecrease value function} suggests an objective function that consists of maximizing the value function and minimizing the difference between the population retention system and the estimated surrogate retention system.
The maximization of the value function is done via solving the PMP (see Algorithm \ref{alg:control solver implementation}),
and optimizing the estimated system is done via collected simulation data.
The central message provided by this Theorem is on the regularization of model parameter updates,
this leads to a modified Hamiltonian of Eq.~\eqref{eq:pmp maximization},
\begin{equation*}
H(t, \boldsymbol\lambda_t, \mat{p}_{t+1}, \boldsymbol\theta_t) \vcentcolon = \mat{p}_{t+1}^{T} \cdot M(\boldsymbol\lambda_t, \boldsymbol\theta_t) - KL(\boldsymbol\theta_t^{\rm old}, \boldsymbol\theta_t),
\end{equation*}
where $\boldsymbol\theta_t^{\rm old}$ represents the predictive model resulting from previous update.
This is consistent with existing RL algorithms \citep{schulman2015trust, schulman2017proximal} as constraining the model parameter update from collected simulation data is beneficial for convergence.

\section{Numerical Experiments}
In this section, we detail two simulation environments to implement the population retention system in Section \ref{sec: simulation environment}. 
Additionally, three categories of baseline methods are discussed: fairness-agnostic, fairness-aware, and dynamic-aware, as presented in Section \ref{sec: Baseline Methods on Asymptotically Fair Participation}. 
Following that, we conduct an empirical validation of our optimal control solution using a synthetic dataset in Section \ref{sec:numerical experiment linear model}. 
We also explore two realistic datasets commonly used in fairness research, as outlined in Section \ref{sec:numerical experiment real-world datasets}.

\subsection{A Generic Platform for Fairness in Non-Stationary Environment}
\label{sec: simulation environment}
In all simulation environments,
we assume a positive association between the proportion of new users at each time point and the current population density. 
Simply put, a higher active user count attracts even more new users to the platform. 
This ensures an increase in population density as the performance of the model improves.
We developed two population retention systems.
\begin{itemize}
    \item 
    $M^{\ast}_1$:
    In this system,
    when a user decides to retain or churn, their decision follows 
    a Bernoulli distribution conditioned on the model performance of this user. 
    For instance, a user has a higher chance of retaining conditioned on correct model prediction, and low probability of staying engaged when a wrong model prediction is given. 
    \item 
    $M^{\ast}_2$:
    The second system takes a more complex approach to modeling user retention. 
    A user decides to churn out because the model has consecutively delivered several inaccurate predictions, for instance, three accumulated incorrect predictions, that particularly pertain to that individual. 
\end{itemize}

\begin{algorithm}[tb]
   \caption{Implementation of Population Retention System for Evaluation.}
   \label{alg: design of population retention system}
\begin{algorithmic}
    \State {\bfseries Input:} 
    Dataset $\{\mat{x}_n, y_n, z_n\}_{n=1}^N$, \\
    A sequence of models $\{ \boldsymbol\theta_t \}_{t=0}^{T-1}$,\\
    Initial population densities $\boldsymbol\lambda_0$.
    \State {\bfseries Output:}
    Population densities at all time step $\{ \boldsymbol\lambda_t \}_{t=0}^T$.
    \For{episode = 1 to max episodes}
    \State {\color{blue} // Set up an initial state by randomly sampling user indices as participative users.}
    \State Initialize $S_0$.
    \For{t = 1 to max time steps}
    \State {\color{blue} // Return the features of currently participative users.}
    \State Return $\mat{x}_n$ if $[S_t]_n = 1$.
    \State {\color{blue} // Model prediction based on the collected features.}
    \State Predict $\boldsymbol\theta_t(\mat{x}_n)$ if $[S_t]_n = 1$.
    \State {\color{blue} // Environment update for active users based on model predictions.}
    \State $S_{t+1} \sim M^{\ast}(\cdot | S_t, \{\boldsymbol\theta_t(\mat{x}_n)\}_{n=1}^N, \{y_n\}_{n=1}^N, \{z_n\}_{n=1}^N)$.
    \State {\color{blue} // Collect reward.}
    \State $R(S_t)$.
    \EndFor
    \EndFor
\end{algorithmic}
\end{algorithm}
Algorithm \ref{alg: design of population retention system} shows the implementation of the population retention system for evaluation.

\subsection{Baseline Algorithms on Asymptotically Fair Participation}
\label{sec: Baseline Methods on Asymptotically Fair Participation}
In this section, we delve into three categories of algorithms designed for achieving asymptotically fair participation: fairness-agnostic, fairness-aware, and dynamic-aware approaches. 
While fairness-agnostic algorithms employ empirical risk minimization, fairness-aware techniques utilize demographic data to ensure balanced model performance across diverse demographic groups. 
Additionally, dynamic-aware approaches consider the inherent population dynamics, typically leading to enhanced performance compared to the other types.

\paragraph{Fairness-agnostic:}
Empirical risk minimization (\textbf{ERM}) optimizes an average loss of all observable data,
\begin{equation*}
\boldsymbol\theta_t
=
\arg \min_{\boldsymbol\theta}
\frac{1} {\sum_{n=1}^N [S_t]_n} \sum_{n=1}^N
\Phi(\boldsymbol\theta_t, \mat{x}_n, y_n)
\cdot [S_t]_n,
\end{equation*}
where $S_t \in \{0, 1\}^N$ indicating the participative user indices (see Section \ref{sec: problem description}).
This method often results in undesirable outcomes in terms of fairness \citep{hashimoto2018fairness}.
Specifically, when there exists a disparity between the population densities of the majority and minority demographic groups at a particular time step,
ERM focuses on minimizing the average loss across all samples. Consequently, this produces a model that performs better for the majority group than for the minority group.
This not only accentuates the imbalance between the two demographic groups in subsequent time steps but also exacerbates the inherent bias in the model derived in the following steps.

\paragraph{Fairness-aware:}
Methods that prioritize fairness use demographic data to ensure equal model performance among various demographic groups. 
We detail two techniques from this group: Minimax optimization (\textbf{Minimax}), and distributional robust optimization (\textbf{DRO}).

\begin{itemize}
    \item 
    \textbf{Minimax} focuses on optimizing the most unfavorable outcome for all groups by using the demographic information of each sample \citep{diana2021minimax}.
    \begin{equation*}
    \boldsymbol\theta_t
    =
    \arg \min_{\boldsymbol\theta}
    \max_{i=[1,2,...,K]}
    \Big(
    \frac{1} {\sum_{n=1}^N [S_t]_n \cdot \mathbbm{1}_{z_n = i}} \sum_{n=1}^N  \Phi(\boldsymbol\theta_t, \mat{x}_n, y_n) \cdot [S_t]_n \cdot \mathbbm{1}_{z_n = i}
    \Big),
    \end{equation*}
    where $\mathbbm{1}_{z_n = i}$ is a indicator function that is used to select samples belonging to the $\rm i^{th}$ demographic group. 
    \item 
    \textbf{DRO} can be seen as a milder form of Minimax since it optimizes for the most unfavorable outcome over a specified proportion (represented by $\lambda$) of samples. 
    Notably, the loss from DRO is always greater than or equal to the loss from Minimax.
    \begin{equation*}
    \boldsymbol\theta_t = \arg \sup_{\mathcal{Q} \in \mathcal{B}(\mathcal{M}^i, r_t^i)} \mathbb{E}_{(\mat{x}, y) \sim \mathcal{Q}} \Phi(\boldsymbol\theta_t, \mat{x}, y), 
    \;\; 
    r_t^i = (1 / \lambda_t^i - 1)^2.
    \end{equation*}
    The DRO algorithm is implemented by the dual representation shown in Eq.~\eqref{eq: dual of worst-case distribution loss}.
\end{itemize}

\paragraph{Dynamic-aware:}
This category takes into account the underlying population evolution for optimal decision-making.
In general time-evolving environments, optimizing model performance at each time step often cannot lead to the optimal model subject to dynamic change.
In the MDP defined in Eq.~\eqref{eq: population retention system} and detailed in Section~\ref{sec: problem description},
the transition of users' behavior can lead to different decision-making algorithms.
We implement various reinforcement learning (RL) algorithms to build dynamic-aware models.

To begin with,
we first define the probability of transitioning from an initial state $S_0$ to any state $S$ under model $\{\boldsymbol\theta_t\}_{t=0}^{T-1}$ and system $M$ for $1$ step (we use $M^{\ast}$ to indicate the population retention system defined in Eq.~\eqref{eq: population retention system} and $M$ as its estimated surrogate retention system defined in Eq.~\eqref{eq: surrogate retention system}),
\begin{equation*}
P(S_0 \rightarrow S, 1, \boldsymbol\theta, M) = \int_A \boldsymbol\theta(A | S_0) M(S | S_0, A) dA, 
\end{equation*}
where $\boldsymbol\theta$ indicates a sequence of models $\{\boldsymbol\theta_t\}_{t=0}^{T-1}$.
More generally,
this transitioning probability admits a recursive form with any steps $t$,
a $t$-step probability transition can be represented as first transitioning to some intermediate state $S'$ after $t-1$ steps, then transitioning to $S$ for one more step,
\begin{equation*}
P(S_0 \rightarrow S, t, \boldsymbol\theta, M) 
= \int_{S'} P(S_0 \rightarrow S', t-1, \boldsymbol\theta, M) \int_A \boldsymbol\theta(A | S') M(S | S', A) dA dS'.
\end{equation*}
We define $\rho^{\boldsymbol\theta, M}$ the stationary state distribution of the MDP under model $\{\boldsymbol\theta_t\}_{t=0}^{T-1}$ and system $M$,
\begin{equation*}
\rho^{\boldsymbol\theta, M}
=
\int_{S_0} \mu(S_0) 
\int_{S} \sum_{t=0}^{\infty} \gamma^t \cdot P(S_0 \rightarrow S, t, \{\boldsymbol\theta_t\}_{t=0}^{T-1}, M).
\end{equation*}

We detail three RL algorithms: Naive policy gradient (\textbf{PG}),
trust region policy optimization (\textbf{TRPO}),
and proximal policy optimization (\textbf{PPO}).

\begin{itemize}
    \item 
    \textbf{PG} \citep{sutton1999policy} is implemented via Monte-Carlo sampling to generate the accumulated reward,
    \begin{equation*}
    \boldsymbol\theta_t = \arg \min_{\boldsymbol\theta} 
    \mathbb{E}_{S \sim \rho^{\boldsymbol\theta, M^{\ast}}}
    \bigg[
    \mathbb{E}_{A \sim \boldsymbol\theta(A|S)}
    \Big[
    G_t \nabla_{\boldsymbol\theta} \ln \boldsymbol\theta(A | S)
    \Big]
    \bigg],
    \end{equation*}
    where $\rho^{\boldsymbol\theta, M^{\ast}}$ is the stationary state distribution resulting from model $\{\boldsymbol\theta_t\}_{t=0}^{T-1}$ and the population retention system $M^{\ast}$,
    $G_t = \sum_{\tau=t}^T \gamma^{\tau-t} R(S_t, A_t)$ represents the accumulated reward from time step $t$,
    which is generated from collecting sample trajectories using the current model $\{\boldsymbol\theta_\tau\}_{\tau=t}^{T-1}$.
    This policy gradient implementation is sample-inefficient since every parameter update requires re-collecting the sample trajectory (the stationary state distribution depends on the current model $\{\boldsymbol\theta_t\}_{t=0}^{T-1}$).
    \item 
    \textbf{TRPO} \citep{schulman2015trust} makes updates that improve the model parameters while ensuring the new model doesn't deviate too much from the old one,
    it can produce more stable and reliable learning compared to vanilla policy gradient methods. 
    \begin{equation*}
    \boldsymbol\theta_t = \arg \min_{\boldsymbol\theta}
    \mathbb{E}_{S \sim \rho^{\boldsymbol{q}, M^{\ast}}}
    \bigg[
    \mathbb{E}_{A \sim \boldsymbol{q}(A | S)}
    \Big[
    \frac{\boldsymbol\theta(A | S)} {\boldsymbol{q}(A | S)}
    \Big]
    G_t
    \bigg],
    \end{equation*}
    where $\boldsymbol{q} = \{\boldsymbol{q}_t\}_{t=0}^{T-1}$ is the sequence of models at previous update.
    \item 
    \textbf{PPO} \citep{schulman2017proximal} aims to approximate the behavior of TRPO but in a more straightforward and computationally efficient manner.
    \begin{equation*}
    \boldsymbol\theta_t = \arg \min_{\boldsymbol\theta}
    \mathbb{E}_{\substack{S \sim \rho^{\boldsymbol{q}, M^{\ast}} \\ A \sim \boldsymbol{q}(A | S)}}
    \Big[
    \min \Big(
    \frac{\boldsymbol\theta(A | S)} {\boldsymbol{q}(A | S)},
    {\rm clip(\frac{\boldsymbol\theta(A | S)} {\boldsymbol{q}(A | S)}, 1 - \epsilon, 1 + \epsilon)}
    \Big)
    G_t
    \bigg],
    \end{equation*}
    where ${\rm clip(\cdot)}$ is the clip function.
    PPO simplifies and improves upon the TRPO method, offering a balance between ease of implementation and sample efficiency.
\end{itemize}


\subsection{Modeling with Synthetic Dataset}
\label{sec:numerical experiment linear model}


In this study, we employ both $M^{\ast}_1$ and $M^{\ast}_2$ defined in Section \ref{sec: simulation environment}, with a synthetic binary classification dataset.
As depicted in Fig.~\ref{fig: structure} (a),
the synthetic dataset is composed of two Gaussian blobs, each centered at different locations of a $2$-dimensional space, 
which formulate the feature distributions of two demographic groups. 
The blobs located on the left and right are denoted as the majority and minority demographic groups, respectively, with respective initial population densities of $0.6$ and $0.4$ (e.g. majority demographic group has a larger population density compared with the minority demographic group at the initial step).
All experiments are repeated with five random seeds.


\paragraph{Quantitative analysis:}
\begin{figure}[t]
\centering
\begin{minipage}{0.8\linewidth}
\centering
\includegraphics[width=\linewidth]{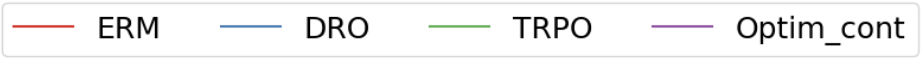}
\end{minipage}

\centering
\begin{minipage}{.27\linewidth}
\centering
\includegraphics[width=\linewidth]{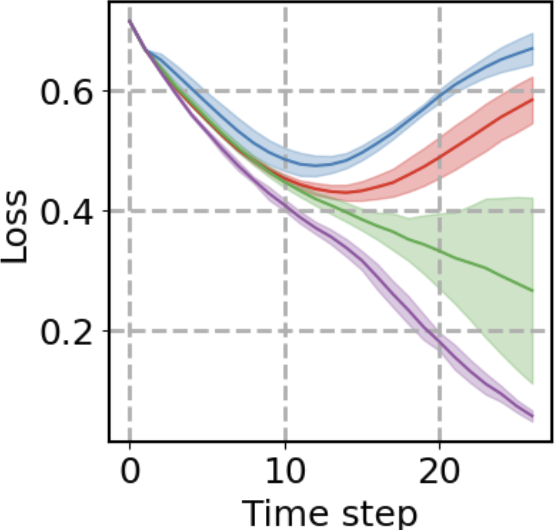}
(a): Loss in $\mathcal{P}_1$
\end{minipage}
\begin{minipage}{.27\linewidth}
\centering
\includegraphics[width=\linewidth]{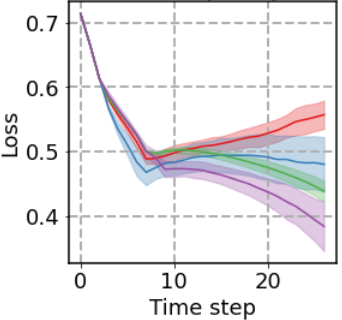}
(b): Loss in $\mathcal{P}_2$
\end{minipage}
\begin{minipage}{.29\linewidth}
\centering
\includegraphics[width=\linewidth]{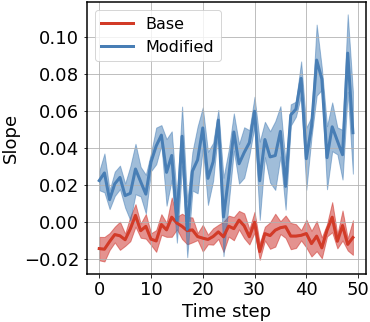}
(c): Slopes
\end{minipage}
\caption{
(a) and (b) plot the binary cross-entropy losses of population densities resulting from ERM, DRO, TRPO, and the proposed Optimal control method in simulation environments $M_1^{\ast}$ and $M_2^{\ast}$ respectively.
(c) plots model slopes resulting from the Optimal control method in the base and modified environments.
}
\label{fig: statistics synthetic dataset}
\end{figure}

We measure a trajectory of population densities (for example, $\boldsymbol\lambda_0$), $\boldsymbol\lambda_1$, ..., $\boldsymbol\lambda_T$) using the binary cross-entropy loss function (e.g. $\sum_{i=1}^K - \log (\lambda_t^i)$). 
To illustrate, during a specific time step, this measurement evaluates how closely a set of population densities approximates a vector where all values are $1$. 
This is based on the premise that the highest possible density for any demographic group is $1$. 
Deviations in population density significantly below $1$ are penalized by this metric.

The plots in Figure~\ref{fig: statistics synthetic dataset} illustrate the loss measures, focusing on the top-performing algorithm from each group: fairness-agnostic, fairness-aware, dynamic-aware, and optimal control methods. 
Specifically, in the simulation environment $M^{\ast}_1$, user churn is sensitive to model accuracy, as a single incorrect prediction can lead to churn with high probability. 
This sensitivity is depicted by the sharp increase in the loss trajectory when using ERM, as seen in Figure~\ref{fig: statistics synthetic dataset} (a). 
The DRO method, which is static and fairness-aware, fails to correct this undesired trend. 
In contrast, TRPO accounts for population dynamics and markedly improves upon the static methods, consistently boosting the population densities of both groups. 
The optimal control method, however, excels by optimally moderating the model across demographics, thereby substantially increasing minority densities with minimal impact on the majority. 
This is evaluated by the loss metrics in Figure~\ref{fig: statistics synthetic dataset} (a), where the optimal control method consistently shows lower losses at every step when compared to other methods.
Experimental results from simulation environment $M^{\ast}_2$, in Figures~\ref{fig: statistics synthetic dataset} (b), show comparable trends,
with smoother transitions.
This is due to the slower variation in user churn behavior specific to the environment $M^{\ast}_2$ (e.g. $3$ wrong prediction causes a user churn with a certain probability).
Moreover, Table \ref{table: terminal stats synthetic dataset} summarizes the terminal conditions of all baseline algorithms in both $M^{\ast}_1$ and $M^{\ast}_2$,
where the terminal population density of the minority group (Density-$2$, higher is better), the disparity between the two groups (Disparity, lower is better), and the terminal loss (Loss, lower is better) are presented.

\begin{table}[t]
\caption{Synthetic Dataset: Terminal State with Initial $\lambda_0^1=0.6$ and $\lambda_0^2=0.4$}
\begin{center}
\begin{tabular}{ p{2.cm}|p{1.2cm} p{1.5cm} p{1.2cm} p{1.2cm} p{1.2cm} p{1.2cm} p{1.2cm} }
\specialrule{1.5pt}{0pt}{0pt}
\multicolumn{8}{c}{Environment $M^{\ast}_1$} \\
\hline
& \multicolumn{1}{c}{Fair-agnostic} & \multicolumn{2}{c}{Fair-aware} & \multicolumn{3}{c}{Dynamic-aware} \\
\hline
& ERM & Minimax & DRO & PG & TRPO & PPO & Optim \\
\hline
Density-2 $\uparrow$ & 0.323 & 0.224 & 0.275 & 0.481 & 0.645 & 0.405 & \textcolor{blue}{0.920} \\
\hline
Disparity $\downarrow$ & 0.677 & 0.776 & 0.715 & 0.459 & 0.335 & 0.555 & \textcolor{blue}{0.05} \\
\hline
Loss $\downarrow$ & 0.56 & 0.76 & 0.66 & 0.40 & 0.23 & 0.47 & \textcolor{blue}{0.06} \\

\specialrule{1.5pt}{0pt}{0pt}
\multicolumn{8}{c}{Environment $M^{\ast}_2$} \\
\hline
& \multicolumn{1}{c}{Fair-agnostic} & \multicolumn{2}{c}{Fair-aware} & \multicolumn{3}{c}{Dynamic-aware} \\
\hline
& ERM & Minimax & DRO & PG & TRPO & PPO & Optim \\
\hline
Density-2 $\uparrow$& 0.316 & 0.215 & 0.348 & 0.312 & 0.327 & 0.314 & \textcolor{blue}{0.471} \\
\hline
Disparity $\downarrow$& 0.684 & 0.785 & 0.652 & 0.688 & 0.673 & 0.686 & \textcolor{blue}{0.529} \\
\hline
Loss $\downarrow$& 0.556 & 0.775 & 0.478 & 0.555 & 0.434 & 0.553 & \textcolor{blue}{0.402} \\
\specialrule{1.5pt}{0pt}{0pt}
\end{tabular}
\end{center}
\label{table: terminal stats synthetic dataset}
\end{table}

The advantage of trajectory planning that acknowledges the underlying dynamics is explored herein.
The idea is to interpret the evolution of the population retention system at certain time steps and observe corresponding adjustments made in the optimal control solution. 
The optimal control method makes performance tradeoffs from the majority group to balance the population densities of both groups at the terminal step.
We manually introduce a substantial quantity of users into the minority demographic group at $t=50$.
Consequently, it is expected that the optimal control method would make fewer tradeoffs and adjust its decision boundaries accordingly at earlier time steps (e.g., $t < 50$).
Referring to Figure \ref{fig: structure} (a), we consider a linear classifier where a positive (resp. negative) slope indicates a model favoring the majority demographic (resp. minority) group.
Figure \ref{fig: statistics synthetic dataset} (c) illustrates the slopes of the model decision boundary at $t \in [0, 50]$. 
As observed, the introduction of additional users to the minority group at $t=50$ enables the optimal control solution to make less performance tradeoff against the majority group due to the increased population density at a later time step.
This adjustment cannot be accomplished by the existing baselines.

\subsection{Modeling with Adult Income and COMPAS Recidivism Racial Bias.}
\label{sec:numerical experiment real-world datasets}
\begin{table}[t]
\caption{Adult Income: Terminal State with Initial $\lambda_0^1=0.6$ and $\lambda_0^2=0.4$}
\begin{center}
\begin{tabular}{ p{2.cm}|p{1.2cm} p{1.5cm} p{1.2cm} p{1.2cm} p{1.2cm} p{1.2cm} p{1.2cm} }
\specialrule{1.5pt}{0pt}{0pt}
\multicolumn{8}{c}{Environment $M^{\ast}_1$} \\
\hline
& \multicolumn{1}{c}{Fair-agnostic} & \multicolumn{2}{c}{Fair-aware} & \multicolumn{3}{c}{Dynamic-aware} \\
\hline
& ERM & Minimax & DRO & PG & TRPO & PPO & Optim \\
\hline
Density-2 $\uparrow$ & 0.292 & 0.252 & 0.291 & 0.284 & 0.285 & 0.285 & \textcolor{blue}{0.318} \\
\hline
Disparity $\downarrow$ & 0.708 & 0.748 & 0.709 & 0.716 & 0.715 & 0.715 & \textcolor{blue}{0.682} \\
\hline
Loss $\downarrow$ & 0.615 & 0.690 & 0.618 & 0.630 & 0.627 & 0.627 & \textcolor{blue}{0.573} \\

\specialrule{1.5pt}{0pt}{0pt}
\multicolumn{8}{c}{Environment $M^{\ast}_2$} \\
\hline
& \multicolumn{1}{c}{Fair-agnostic} & \multicolumn{2}{c}{Fair-aware} & \multicolumn{3}{c}{Dynamic-aware} \\
\hline
& ERM & Minimax & DRO & PG & TRPO & PPO & Optim \\
\hline
Density-2 $\uparrow$ & 0.363 & 0.315 & 0.353 & 0.334 & 0.334 & 0.334 & \textcolor{blue}{0.397} \\
\hline
Disparity $\downarrow$ & 0.637 & 0.685 & 0.647 & 0.666 & 0.666 & 0.666 & \textcolor{blue}{0.603} \\
\hline
Loss $\downarrow$ & 0.507 & 0.579 & 0.521 & 0.549 & 0.548 & 0.549 & \textcolor{blue}{0.462} \\
\specialrule{1.5pt}{0pt}{0pt}
\end{tabular}
\end{center}
\label{table: terminal stats adule income}
\end{table}

\begin{table}[t]
\caption{COMPAS: Terminal State with Initial $\lambda_0^1=0.6$ and $\lambda_0^2=0.4$}
\begin{center}
\begin{tabular}{ p{2.cm}|p{1.2cm} p{1.5cm} p{1.2cm} p{1.2cm} p{1.2cm} p{1.2cm} p{1.2cm} }
\specialrule{1.5pt}{0pt}{0pt}
\multicolumn{8}{c}{Environment $M^{\ast}_1$} \\
\hline
& \multicolumn{1}{c}{Fair-agnostic} & \multicolumn{2}{c}{Fair-aware} & \multicolumn{3}{c}{Dynamic-aware} \\
\hline
& ERM & Minimax & DRO & PG & TRPO & PPO & Optim \\
\hline
Density-2 $\uparrow$ & 0.271 & 0.256 & 0.184 & 0.263 & 0.264 & 0.257 & \textcolor{blue}{0.274} \\
\hline
Disparity $\downarrow$ & 0.522 & 0.529 & \textcolor{blue}{0.153} & 0.435 & 0.489 & 0.442 & 0.516 \\
\hline
Loss $\downarrow$ & 0.770 & 0.803 & 1.391 & 0.848 & 0.808 & 0.858 &\textcolor{blue}{0.764} \\

\specialrule{1.5pt}{0pt}{0pt}
\multicolumn{8}{c}{Environment $M^{\ast}_2$} \\
\hline
& \multicolumn{1}{c}{Fair-agnostic} & \multicolumn{2}{c}{Fair-aware} & \multicolumn{3}{c}{Dynamic-aware} \\
\hline
& ERM & Minimax & DRO & PG & TRPO & PPO & Optim \\
\hline
Density-2 $\uparrow$ & 0.316 & 0.298 & 0.075 & \textcolor{blue}{0.317} & \textcolor{blue}{0.317} & \textcolor{blue}{0.317} & \textcolor{blue}{0.317} \\
\hline
Disparity $\downarrow$ & 0.684 & 0.702 & 0.925 & 0.863 & \textcolor{blue}{0.683} & \textcolor{blue}{0.683} & \textcolor{blue}{0.683} \\
\hline
Loss $\downarrow$ & 0.577 & 0.605 & 1.296 & 0.594 & \textcolor{blue}{0.574} & \textcolor{blue}{0.574} & \textcolor{blue}{0.574} \\
\specialrule{1.5pt}{0pt}{0pt}
\end{tabular}
\end{center}
\label{table: terminal stats compas}
\end{table}

We explore two real-world datasets: the Adult Income dataset (\textbf{Adult}) \citep{kohavi1996scaling} and the \textbf{COMPAS} dataset \citep{barenstein2019propublica}. 
The Adult dataset provides information on individual annual incomes, influenced by factors like gender, race, age, and education.
COMPAS, on the other hand, is a commercial tool used in the legal system to predict a criminal defendant's likelihood of reoffending. 
In both datasets, gender attributes are used to differentiate demographic groups.
To simulate population dynamics in these static datasets, we apply the population retention system outlined in Eq.\eqref{eq: population retention system}. 
We follow the same simulation configurations as the experiment with the synthetic dataset in Section~\ref{sec:numerical experiment linear model}. 
For each demographic group, we randomly select $N^i=1000$ samples and set initial population densities at $\lambda_0^1=0.6$ for the majority group and $\lambda_0^2=0.4$ for the minority group.

The outcomes of these simulations, specifically for the Adult dataset, are summarized in Table \ref{table: terminal stats adule income},
in which the minority population density at the terminal state (Density-$2$, higher is better), the disparity between two population densities (e.g. $|\lambda_T^1 - \lambda_T^2|$, lower is better),
and the loss measures (Loss, lower is better) are presented.
These results show that our proposed optimal control method (Optim) outperforms other baseline methods. 
In environments $M_1^{\ast}$ and $M_2^{\ast}$, it achieves terminal states with $\lambda_T^1=1.0$ and $\lambda_T^2=0.318$, and $\lambda_T^1=1.0$ and $\lambda_T^2=0.81$, respectively.

The results from the COMPAS dataset are detailed in Fig.~\ref{table: terminal stats compas},
where minority population density, disparity, and loss measures are shown.
In this scenario, the optimal control method shows comparable performance to RL-based algorithms. 
Specifically,
TRPO and PPO reach a terminal state with a loss value of $0.574$, which achieves the same level of performance compared to the proposed optimal control method.
This similarity in performance is attributed to the lesser disparity in representation between different gender attributes within the COMPAS dataset.

\section{Related Works}
\label{sec: related works}
We delve into the existing body of literature surrounding fairness in non-stationary settings. 
Moreover, we explore the realm of machine learning from an optimal control perspective, emphasizing its relevance and applicability to the task at hand.

\paragraph{Fairness problems in the non-stationary setting.}
Emerging research has illuminated the complexities associated with imposing static fairness constraints in machine learning models \citep{hardt2016equality, feldman2015certifying}. 
While these constraints aim to ensure equal treatment of different demographic groups, they can inadvertently introduce undesirable long-term effects, as detailed in studies by \citep{liu2018delayed, zhang2020fair}. One central concern is that algorithmic fairness considerations in static settings do not adequately account for real-world environments.
The feedback loop between algorithmic decisions and individuals' reactions is particularly noteworthy in this context \citep{zhang2020fair}. 
As algorithms make decisions, individuals respond based on those decisions, thereby altering the original data distribution the model was trained on. 
For example, a model's decision could lead to behavioral changes in individuals, leading to shifts in the underlying data distribution. 
This subsequently has repercussions on the model's performance in future iterations. Such a phenomenon was highlighted in a study by \citep{zhang2019group}, which undertook an in-depth analysis of how user retention rates interplay with model decisions in dynamic environments.
Adding another layer of complexity, the optimization techniques used to ensure fairness are often tailored for stationary settings. 
These methods, which frequently employ successive one-step approaches, might prioritize immediate fairness for minority demographic groups without considering potential long-term effects \citep{hashimoto2018fairness}. 
In light of these findings, our study underscores the imperative to view fairness not just as a static ideal but as a dynamic equilibrium that respects the ever-evolving nature of real-world contexts. 
Understanding and incorporating these dynamics into fairness considerations will be vital for building fair machine learning systems.

\paragraph{The connection between machine learning and optimal control.}
The connections between dynamical systems and deep learning have been a focal point of recent studies, drawing attention to their underlying relationships \citep{weinan2017proposal, haber2017stable}. 
This approach provides a theoretical framework that reinterprets deep learning methodologies through optimal control \citep{liu2019deep}. 
The contributions of \cite{JMLR:v18:17-653} and \cite{li2018optimal} connect the traditional back-propagation algorithm with the optimal control theory. 
This synthesis illustrated the profound connection between Pontryagin's Maximum Principle \citep{kirk1970optimal} and gradient-based training procedures in neural networks.
Along this path, \citet{weinan2018mean} further developed the theoretical underpinnings about the interpretation of deep learning from an optimal control viewpoint.
Their efforts laid down rigorous mathematical foundations.
Optimal control techniques have also been demonstrated to address some of the most pressing challenges in the realm of deep learning. 
A noteworthy contribution is from \citep{liu2020differential}, who introduced efficient high-order optimization techniques grounded in differential dynamic programming. 
This approach has shown promise in improving training convergence and stability. 
Furthermore, the research by \citep{chen2021towards} embarked on the development of closed-loop controllers tailored to bolster a model's resilience against adversarial perturbations.
In a related vein, \citep{dupont2019augmented} explored the application of ordinary differential equations in deep residual networks, drawing connections between the evolution of deep networks and the trajectory of dynamical systems. 

\section{Limitations and Future Works}
\label{sec: limitation and future works}
In this section, we provide a discussion of the limitations of the proposed framework, thereby highlighting potential directions for improvement and development.

\paragraph{Constraints on intermediate population densities:}
In discussions surrounding fairness in machine learning and algorithms, the concept of asymptotic behavior becomes crucial. 
Specifically, it provides insights into how different groups or populations evolve in the long run, effectively delineating their long-term densities. 
This perspective, while valuable, poses challenges when considered in isolation. 
Sole reliance on asymptotically fair participation can overlook disparities and biases that may manifest during intermediate phases of the algorithm's operation. 
Such oversight can lead to situations where certain groups or users experience disparities, even if the long-term projections seem fair. 
This can result in not only statistical discrepancies but also suboptimal user experiences. 
Recognizing this potential pitfall, our upcoming work proposes an enhancement to the current fairness framework by introducing an additional running loss. 
This modification aims to provide a more realistic view of fairness, accounting for both intermediate and long-term behaviors, thereby ensuring a more comprehensive and fair system.


\paragraph{Computational efficiency:}
In our existing framework, we solve the optimal control problem by leveraging PMP.
PMP inherently requires the optimization of the Hamiltonian dynamics for each distinct initial condition. 
In the current work, we've employed the method of successive approximation for efficient algorithmic implementation. 
While the results from our current experiments validate the robustness of this approach, there is a concern related to its computational scalability. As models grow in size and datasets become larger, the computational demands intensify. 
The intricacies and computational overhead of the PMP become particularly evident in these dynamic environments.
To address the challenge of computational complexity,
we propose to approximate the surrogate retention system via linear approximation, a process that promises a closed-form solution to the optimal control formulation.
However,
the current surrogate retention system possesses a high non-linearity. 
Given these challenges, our subsequent goal is to construct a surrogate retention system that allows for accurate linearization,
in which a closed-form solution can be constructed.

\paragraph{Connections with model-based reinforcement learning.}
Model-based RL techniques are increasingly drawing attention due to their inherent strengths, particularly in the realm of sample complexity. 
Sample complexity, the number of samples or experiences an agent needs to learn an effective policy, has been a consistent bottleneck for several RL algorithms. 
Model-free RL methods, especially those rooted in policy gradient-based approaches, often suffer from this issue, leading to longer training times.
However, it's crucial to recognize that while model-based RL techniques present a solution to the sample complexity dilemma, they come with their own set of challenges. 
For instance,
effectively scaling these methods to address high-dimensional problems with a vast state or action space. 
The intricacies of representing, planning, and learning in such domains can be computationally demanding and often lead to suboptimal policies.
In this work, instead of attempting to model the high-dimensional space, we propose treating the population density as a statistical average. 
This perspective allows us to abstract away some of the complexities and capture the essence of the underlying dynamics and evolutions within the system. 
As a result, we can transform the original high-dimensional problem into a more manageable, low-dimensional surrogate retention system. 
This simplification, while retaining the core dynamics, makes the learning process more tractable and efficient.

\section{Conclusion}
With the growing integration of machine learning systems into real-world applications, studying long-term fairness problems holds substantial significance. 
In this work, we have presented a framework to investigate fairness in a dynamic setting, where user participation is conditioned on model performance. 
We have defined asymptotic fairness to evaluate fairness within this dynamic context and formulated it as an optimal control problem. 
To accomplish this, we have constructed a surrogate retention system that provides estimates of the underlying environment. 
The proposed optimal control solution has been demonstrated to be effective through simulations. 
As a potential avenue for future work, we aim to develop more sophisticated surrogate systems that can handle intricate environments, potentially involving real-world human users.




\newpage

\appendix

\section{Monotone Improvement}
\label{sec: derivation for nondecreasing value function}
In this section, we provide derivation for Theorem \ref{theorem: nondecrease value function}. 
We consider an infinity horizon discounted reward setting,
where the reward function $R(S_t)$ is defined over a state.
As defined in Section \ref{sec: problem description},
the state $S_t$ includes indices of participative and non-participative  users,
the reward function can be defined as measuring the population density of participative users,
for instance, 
\begin{equation*}
R(S_t) = \sum_{i=1}^N \lambda_t^i
\;\;
{\rm where}
\;\;
\lambda_t^i = \frac{1} {N^i} \cdot \sum_{n=1}^N [S_t]_n \cdot \mathbbm{1}_{z_n = i},
\end{equation*}
where this reward function measures the sum of all population densities at time step $t$.
Moreover,
we use $M^{\ast}$ to indicate the population retention system defined in Eq.~\eqref{eq: population retention system} and $M$, $M'$ as its estimations.
We denote $V^{\boldsymbol\theta, M}$ as the value function of predictive model $\boldsymbol\theta = \{\boldsymbol\theta_t\}_{t=0}^{T-1}$ and dynamical system $M$,
\begin{equation*}
V^{\boldsymbol\theta, M}(s)
=
\mathbb{E}_{S_{t+1} \sim M(S_{t+1} | S_t, \boldsymbol\theta(S_t))}
\Big[
\sum_{t=0}^{\infty} \gamma^t R(S_t) | S_0 = s
\Big],
\end{equation*}
where the predictive model $\boldsymbol\theta$ generate deterministic outcomes,
$\gamma$ is a discounting factor.
The proof is structured as follows,
\begin{itemize}
    \item 
    We first calculate the difference between the value functions of two distinct dynamical systems,
    $V^{\boldsymbol\theta, M} - V^{\boldsymbol\theta, \hat{M}}$ (See Lemma \ref{lemma: discrepancy between two value functions}).
    \item 
    We assume that the value function $V^{\boldsymbol\theta, M}$ satisfies an L-Lipschitz condition to a certain norm, 
    and determine an upper bound for the difference between the value functions resulting from the population retention system and an estimated dynamical system (e.g. surrogate retention system) (See Proposition \ref{proposition: value function upper bound naive}).
    \item 
    We discuss the challenge of optimizing the aforementioned upper bound. 
    To address this, we further refine this upper bound. (See Proposition \ref{proposition: value function upper bound refine}).
\end{itemize}

To begin with,
we define the probability of transitioning from an initial state $S_0$ to any state $S$ under the predictive model $\boldsymbol\theta$ and dynamical system $M$ for $1$ step,
\begin{equation*}
P(S_0 \rightarrow S, 1, \boldsymbol\theta, M) 
= 
M(S | S_0, \boldsymbol\theta_0(S_0)),
\end{equation*}
this transitioning probability admits a recursive form with any steps $t$.
A $t$-step probability transition can be represented as first transitioning to some intermediate state $S'$ after $t-1$ steps, then transitioning to $S$ for one more step,
\begin{equation*}
P(S_0 \rightarrow S, t, \boldsymbol\theta, M) 
= \int_{S'} P(S_0 \rightarrow S', t-1, \boldsymbol\theta, M) \cdot M(S | S', \boldsymbol\theta(S')) dS'.
\end{equation*}
Moreover,
we define $\rho^{\boldsymbol\theta, M}$ as the stationary state distribution,
\begin{equation}
\label{eq: definition of rho}
\rho^{\boldsymbol\theta, M}
=
\int_{S_0} \mu(S_0) \int_{S} \sum_{t=0}^{\infty} \gamma^t P(S_0 \rightarrow S, t, \boldsymbol\theta, M),
\end{equation}
where $\gamma$ is the discounting factor, and we assume that the initial state $S_0$ follows uniform distribution over the state space (e.g. $\mu(S_0)$ has equal probability for all possible $S_0$).

Let $M$ and $\hat{M}$ be two dynamical systems,
for any predictive model $\{ \boldsymbol\theta_t \}_{t=0}^{T-1}$,
the following lemma derives the discrepancy between $V^{\boldsymbol\theta, M} - V^{\boldsymbol\theta, \hat{M}}$.

\begin{lemma}
\label{lemma: discrepancy between two value functions}
For any predictive model $\boldsymbol\theta$ and two distinct dynamical systems, $M$ and $\hat{M}$, the following holds true,
\begin{equation*}
V^{\boldsymbol\theta, M} - V^{\boldsymbol\theta, \hat{M}}
=
\gamma \cdot
\mathbb{E}_{\substack{S \sim \rho^{\boldsymbol\theta, M}}}
\Big[
\mathbb{E}_{S' \sim M(S' | S, \boldsymbol\theta(S))} 
\big[
V^{\boldsymbol\theta, \hat{M}} (S')
\big]
-
\mathbb{E}_{\hat{S} \sim \hat{M}(\hat{S} | S, \boldsymbol\theta(S))} 
\big[
V^{\boldsymbol\theta, \hat{M}} (\hat{S})
\big]
\Big].
\end{equation*}
\end{lemma}

\begin{proof}
We denote $w_j(S_0)$ as the discounted rewards computed from a dynamical system $M$ for the first $j$ steps and another dynamical system $\hat{M}$ starting from the $\rm {(j+1)}^{th}$ step.
\begin{gather*}
w_j(S_0)
=
\sum_{t=0}^j \gamma^t \int_S P(S_0 \rightarrow S, t, \boldsymbol\theta, M)  \cdot R(S) dS
\\
+
\gamma^{j+1} \int_S P(S_0 \rightarrow S, j, \boldsymbol\theta, M) \int_{S'} \hat{M}(S' | S, \boldsymbol\theta(S)) \cdot V^{\boldsymbol\theta, \hat{M}} (S') dS' dS,
\\
=
\sum_{t=0}^j \gamma^t 
\Big(
\mathbb{E}_{\substack{S \sim P(S_0 \rightarrow S, t, \boldsymbol\theta, M)}}
\big[
R(S)
\big]
\Big)
+
\gamma^{j+1} 
\Big(
\mathbb{E}_{\substack{S \sim P(S_0 \rightarrow S, j, \boldsymbol\theta, M)}}
\big[
\mathbb{E}_{\hat{S} \sim \hat{M}(\hat{S} | S, \boldsymbol\theta(S))} 
\big[
V^{\boldsymbol\theta, \hat{M}} (\hat{S})
\big]
\big]
\Big).
\end{gather*}

Moreover, $w_{j+1}(S_0)$ at time step $t+1$ can be shown similarly,
\begin{align*}
& w_{j+1}(S_0)
\\
& =
\sum_{t=0}^j \gamma^t 
\Big(
\mathbb{E}_{\substack{S \sim P(S_0 \rightarrow S, t, \boldsymbol\theta, M)}}
[
R(S)
]
\Big)
+ 
\gamma^{j+1} 
\Big(
\mathbb{E}_{\substack{S \sim P(S_0 \rightarrow S, j, \boldsymbol\theta, M)}}
\big[
\mathbb{E}_{S' \sim M(S' | S, \boldsymbol\theta(S))} 
\big[
V^{\boldsymbol\theta, \hat{M}} (S')
\big]
\big]
\Big).
\end{align*}

Notice that the difference between $w_j(S_0)$ and $w_{j+1}(S_0)$ lies in the transitioning dynamical system from state $S_j$ to $S_{j+1}$ ($w_j(S_0)$ relies on $\hat{M}$ and $w_{j+1}(S_0)$ relies on $M$).
From the definitions of $w_j(S_0)$ and $w_{j+1}(S_0)$,
the discrepancy between the value functions $V^{\boldsymbol\theta, M} (S_0)$ and $V^{\boldsymbol\theta, \hat{M}} (S_0)$ can be reformulated in terms of the definition of $w_j(S_0)$,
\begin{equation*}
V^{\boldsymbol\theta, M} (S_0) - V^{\boldsymbol\theta, \hat{M}} (S_0)
=
\sum_{j=0}^{\infty} w_{j+1}(S_0) - w_j(S_0)
=
w_{\infty}(S_0) - w_0(S_0),
\end{equation*}
since $V^{\boldsymbol\theta, M} (S_0) = w_{\infty}(S_0)$,
and $V^{\boldsymbol\theta, \hat{M}} (S_0) = w_0(S_0)$.

For a given $j$, each term $w_{j+1}(S_0) - w_j(S_0)$ can be computed based on their definitions,
\begin{align*}
& w_{j+1}(S_0) - w_j(S_0)
\\
& =
\gamma^{j+1} 
\bigg(
\mathbb{E}_{\substack{S \sim P(S_0 \rightarrow S, j, \boldsymbol\theta, M)}}
\Big[
\mathbb{E}_{S' \sim M(S' | S, \boldsymbol\theta(S))} [V^{\boldsymbol\theta, \hat{M}} (S')]
-
\mathbb{E}_{\hat{S} \sim \hat{M}(\hat{S} | S, \boldsymbol\theta(S))} [V^{\boldsymbol\theta, \hat{M}} (\hat{S})]
\Big]
\bigg).
\end{align*}

Recall that $\rho^{\boldsymbol\theta, M}
=
\int_{S_0} \mu(S_0) \int_{S} \sum_{t=0}^{\infty} \gamma^t P(S_0 \rightarrow S, t, \boldsymbol\theta, M)$,
the expected value of $V^{\boldsymbol\theta, M} (S_0) - V^{\boldsymbol\theta, \hat{M}} (S_0)$ with respect to $S_0$ can be computed as follows,
\begin{align*}
& V^{\boldsymbol\theta, M} - V^{\boldsymbol\theta, \hat{M}}
\\
& = 
\int_{S_0} \mu(S_0) \big[ V^{\boldsymbol\theta, M} (S_0) - V^{\boldsymbol\theta, \hat{M}} (S_0) \big],
\\
& =
\sum_{j=0}^{\infty}
\gamma^{j+1} 
\cdot
\int_{S_0} \mu(S_0) 
\cdot
\mathbb{E}_{\substack{S \sim P(S_0 \rightarrow S, j, \boldsymbol\theta, M)}}
\Big[
\mathbb{E}_{\substack{S' \sim M(S' | S, \boldsymbol\theta(S)) \\ 
\hat{S} \sim \hat{M}(\hat{S} | S, \boldsymbol\theta(S))}} [V^{\boldsymbol\theta, \hat{M}} (S')
-
V^{\boldsymbol\theta, \hat{M}} (\hat{S})]
\Big],
\\
& =
\gamma \cdot
\mathbb{E}_{\substack{S \sim \rho^{\boldsymbol\theta, M}}}
\Big[
\mathbb{E}_{S' \sim M(S' | S, \boldsymbol\theta(S))} 
\big[
V^{\boldsymbol\theta, \hat{M}} (S')
\big]
-
\mathbb{E}_{\hat{S} \sim \hat{M}(\hat{S} | S, \boldsymbol\theta(S))} 
\big[
V^{\boldsymbol\theta, \hat{M}} (\hat{S})
\big]
\Big].
\end{align*}
\end{proof}

\begin{proposition}
\label{proposition: value function upper bound naive}
Suppose that the value function $V^{\boldsymbol\theta, M}$ on any dynamical model $M$ is $L$-Lipschitz with respect to some norm $\lVert \cdot \rVert$ in the sense that
\begin{equation*}
|V^{\boldsymbol\theta, M}(S) - V^{\boldsymbol\theta, M}(S')| \leq L \cdot \lVert S - S' \rVert,
\;\;\;\;
\forall S, S' \in \mathcal{S},
\end{equation*}
and assume that the underlying dynamical system is deterministic,
then the following establishes an upper bound for the discrepancy between the value functions $V^{\boldsymbol\theta, M}$ of an estimated dynamical system $M$ and the value function of the true environment $V^{\boldsymbol\theta, M^{\ast}}$,
\begin{equation*}
|V^{\boldsymbol\theta, M} - V^{\boldsymbol\theta, M^{\ast}}|
\leq
\gamma \cdot L \cdot
\mathbb{E}_{\substack{S \sim \rho^{\boldsymbol\theta, M^{\ast}}}}
\Big[
\lVert M(S, \boldsymbol\theta(S)) - M^{\ast}(S, \boldsymbol\theta(S)) \rVert
\Big].
\end{equation*}
\end{proposition}

\begin{proof}
According to Lemma \ref{lemma: discrepancy between two value functions},
\begin{align*}
& |V^{\boldsymbol\theta, M^{\ast}} - V^{\boldsymbol\theta, M}|
\\
& =
\gamma \cdot
\Big|
\mathbb{E}_{\substack{S \sim \rho^{\boldsymbol\theta, M^{\ast}}}}
\Big[
\mathbb{E}_{S' \sim M^{\ast}(S' | S, \boldsymbol\theta(S))} 
\big[
V^{\boldsymbol\theta, M} (S')
\big]
-
\mathbb{E}_{\hat{S} \sim M(\hat{S} | S, \boldsymbol\theta(S))} 
\big[
V^{\boldsymbol\theta, M} (\hat{S})
\big]
\Big]
\Big|,
\\
& =
\gamma \cdot
\Big|
\mathbb{E}_{\substack{S \sim \rho^{\boldsymbol\theta, M^{\ast}}}}
\Big[
V^{\boldsymbol\theta, M} (M^{\ast}(S, \boldsymbol\theta(S)))
-
V^{\boldsymbol\theta, M} (M(S, \boldsymbol\theta(S)))
\Big]
\Big|,
\;\;
{\rm (Deterministic)}
\\
& \leq
\gamma \cdot
\mathbb{E}_{\substack{S \sim \rho^{\boldsymbol\theta, M^{\ast}}}}
\Big[
|
V^{\boldsymbol\theta, M} (M^{\ast}(S, \boldsymbol\theta(S)))
-
V^{\boldsymbol\theta, M} (M(S, \boldsymbol\theta(S)))
|
\Big],
\\
& \leq
\gamma \cdot
L \cdot
\mathbb{E}_{\substack{S \sim \rho^{\boldsymbol\theta, M^{\ast}}}}
\Big[
\lVert M^{\ast}(S, \boldsymbol\theta(S)) - M(S, \boldsymbol\theta(S)) \rVert
\Big].
\;\;
{\rm (L-Lipschitz)}
\end{align*}
\end{proof}

Given two transition kernels $\mat{P}$ and $\mat{P}'$, 
the accumulated state transitions resulting from $\mat{P}$ and $\mat{P}'$ are 
$\sum_{t=0}^{\infty} (\gamma \mat{P})^t = (I - \gamma \mat{P})^{-1}$,
and
$\sum_{t=0}^{\infty} (\gamma \mat{P}')^t = (I - \gamma \mat{P}')^{-1}$,
respectively.
The subsequent Corollary establishes an upper bound for the difference between two state distributions that emerge from these kernels.
\begin{corollary}
\label{corollary: bound between two distributions}
Let $\mu$ be a distribution over the state space,
$d = (1 - \gamma)(I - \gamma \mat{P})^{-1} \mu$,
and $d' = (1 - \gamma) (I - \gamma \mat{P}')^{-1} \mu$ denote the discounted distribution starting from $\mu$ induced by the transitions $\mat{P}$ and $\mat{P}'$. Then,
\begin{equation*}
|d - d'|_1 \leq \frac{\gamma} {1 - \gamma} |(\mat{P} - \mat{P}') d'|_1.
\end{equation*}
\end{corollary}

\begin{proof}
\begin{align*}
|d - d'|_1 & 
= (1 - \gamma) \cdot 
|(\mat{I} - \gamma \mat{P})^{-1} \mu - (\mat{I} - \gamma \mat{P}')^{-1} \mu|_1,
\\
& =
(1 - \gamma)
\cdot
|\Big(
(\mat{I} - \gamma \mat{P})^{-1}
\Big(
(\mat{I} - \gamma \mat{P}') - (\mat{I} - \gamma \mat{P})
\Big)
(\mat{I} - \gamma \mat{P}')^{-1}
\mu
\Big)|_1,
\\
& = (1 - \gamma)
\cdot
|\Big(
(\mat{I} - \gamma \mat{P})^{-1}
(\gamma \mat{P} - \gamma \mat{P}')
(\mat{I} - \gamma \mat{P}')^{-1}
\mu
\Big)|_1,
\\
& \leq
|
\gamma (\mat{P} - \mat{P}') (\mat{I} - \gamma \mat{P}')^{-1} \mu
|_1,
\\
& =
\frac{\gamma} {1 - \gamma}
|
(\mat{P} - \mat{P}') d'
|_1,
\end{align*}
where 
$(1 - \gamma) \cdot |(\mat{I} - \gamma \mat{P})^{-1}|_1 \leq 1$.
\end{proof}

Consider two sequences of predictive models, denoted as $\boldsymbol\theta = \{\boldsymbol\theta_t\}_{t=0}^{T-1}$ and $\boldsymbol\theta' = \{\boldsymbol\theta_t'\}_{t=0}^{T-1}$. 
The stationary state distributions resulting from $\boldsymbol\theta$ and $\boldsymbol\theta'$ are $\rho^{\boldsymbol\theta, M^{\ast}}$ and $\rho^{\boldsymbol\theta', M^{\ast}}$.
The subsequent Corollary establishes an upper bound for the stationary state distributions that emerge from $\boldsymbol\theta$ and $\boldsymbol\theta'$.
\begin{corollary}
\label{corollary: upper bound between two distributions}
The following holds true for $\rho^{\boldsymbol\theta, M^{\ast}}$ and $\rho^{\boldsymbol\theta', M^{\ast}}$,
\begin{equation*}
|\rho^{\boldsymbol\theta, M^{\ast}} - \rho^{\boldsymbol\theta', M^{\ast}}|_1
\leq
\frac{\gamma} {1 - \gamma}
\cdot
\mathbb{E}_{S \sim \rho^{\boldsymbol\theta', M^{\ast}}}
\Big[
KL(\boldsymbol\theta(S), \boldsymbol\theta'(S))^{\frac{1} {2}}
\Big].
\end{equation*}
\end{corollary}

\begin{proof}
Recall Corollary \ref{corollary: bound between two distributions},
given two state distributions $\rho^{\boldsymbol\theta, M^{\ast}}$ and $\rho^{\boldsymbol\theta', M^{\ast}}$,
\begin{equation*}
|\rho^{\boldsymbol\theta, M^{\ast}} - \rho^{\boldsymbol\theta', M^{\ast}}|_1
\leq \frac{\gamma} {1 - \gamma}
\cdot
\mathbb{E}_{S \sim \rho^{\boldsymbol\theta', M^{\ast}}}
\Big[
|P_{M^{\ast}(S, \boldsymbol\theta(S))}
-
P_{M^{\ast}(S, \boldsymbol\theta'(S))}|_1
\Big],
\end{equation*}
where $P_{M^{\ast}(S, \boldsymbol\theta(S))}$ represents the probability distribution of $M^{\ast}(S, \boldsymbol\theta(S))$.
Since $P_{M^{\ast}(S, \boldsymbol\theta(S))}$ is a mapping from a state action pair to a probability,
\begin{equation*}
\frac{\gamma} {1 - \gamma}
\cdot
\mathbb{E}_{S \sim \rho^{\boldsymbol\theta', M^{\ast}}}
\Big[
|P_{M^{\ast}(S, \boldsymbol\theta(S))}
-
P_{M^{\ast}(S, \boldsymbol\theta'(S))}|_1
\Big]
\leq
\frac{\gamma} {1 - \gamma}
\cdot
\mathbb{E}_{S \sim \rho^{\boldsymbol\theta', M^{\ast}}}
\Big[
|P_{\boldsymbol\theta(S)}
-
P_{\boldsymbol\theta'(S)}|_1
\Big],
\end{equation*}
where $P_{\boldsymbol\theta(S)}$ represents the probability distribution of $\boldsymbol\theta(S)$.
Based on Pinkser’s inequality,
\begin{equation*}
\frac{\gamma} {1 - \gamma}
\cdot
\mathbb{E}_{S \sim \rho^{\boldsymbol\theta', M^{\ast}}}
\Big[
|P_{\boldsymbol\theta(S)}
-
P_{\boldsymbol\theta'(S)}|_1
\Big]
\leq
\frac{\gamma} {1 - \gamma}
\cdot
\mathbb{E}_{S \sim \rho^{\boldsymbol\theta', M^{\ast}}}
\Big[
KL(\boldsymbol\theta(S), \boldsymbol\theta'(S))^{\frac{1} {2}}
\Big],
\end{equation*}
where $KL(\cdot ,\cdot)$ represents the KL divergence of two distributions.
\end{proof}

Recall in Proposition \ref{proposition: value function upper bound naive},
$|V^{\boldsymbol\theta, M} - V^{\boldsymbol\theta, M^{\ast}}|
\leq
\gamma \cdot L \cdot
\mathbb{E}_{\substack{S \sim \rho^{\boldsymbol\theta, M^{\ast}}}}
\Big[
\lVert M(S, \boldsymbol\theta(S)) - M^{\ast}(S, \boldsymbol\theta(S)) \rVert
\Big]$.
In the stationary state distribution $\rho^{\boldsymbol\theta, M^{\ast}}$, the upper bound has explicit dependence on the model parameter $\boldsymbol\theta$. 
However, this complex dependency on $\boldsymbol\theta$ complicates the process of incorporating this upper bound into any objective function for the purpose of optimizing model parameters.
The following Proposition further refines this upper bound to convert the explicit dependence on $\boldsymbol\theta$ to a reference model $\boldsymbol\theta^{ref}$.

\begin{proposition}
\label{proposition: value function upper bound refine}
Assume that the dynamical system is deterministic. 
Consider the value function $V^{\boldsymbol\theta, M}$ for the estimated dynamical model $M$, which is $L$-Lipschitz. 
Also, assume that the state space is uniformly bounded by $B$. 
Under these conditions, we can determine an upper bound for the difference between the value functions $V^{\boldsymbol\theta, M}$ of the estimated dynamical system $M$ and the value function corresponding to the actual environment $M^{\ast}$.
\begin{equation*}
|V^{\boldsymbol\theta, M} - V^{\boldsymbol\theta, M^{\ast}}|
\leq
\gamma \cdot L \cdot
\mathbb{E}_{\substack{S \sim \rho^{\boldsymbol\theta^{\rm ref}, M^{\ast}} }}
\Big[
\lVert M(S, \boldsymbol\theta(S)) - M^{\ast}(S, \boldsymbol\theta(S)) \rVert
\Big]
+
2 B \kappa \frac{\gamma} {1 - \gamma},
\end{equation*}
where $\kappa$ is an upper bound on the KL divergence between $\boldsymbol\theta^{\rm ref}$ and  $\boldsymbol\theta$.
\end{proposition}

\begin{proof}
For any distributions $\rho$ and $\rho'$ and function $f(\cdot)$,
we have 
\begin{align*}
\mathbb{E}_{S \sim \rho} f(S)
& =
\mathbb{E}_{S \sim \rho'} f(S)
+
<\rho - \rho', f>
\\
& 
\leq 
\mathbb{E}_{S \sim \rho'} f(S) + \lVert \rho - \rho' \rVert_1
\cdot
\lVert f \rVert_{\infty}.
\end{align*}
Recall Proposition \ref{proposition: value function upper bound naive} and apply this inequality,
\begin{align*}
& |V^{\boldsymbol\theta, M} - V^{\boldsymbol\theta, M^{\ast}}|
\\
& \leq
\gamma \cdot L \cdot
\mathbb{E}_{\substack{S \sim \rho^{\boldsymbol\theta, M^{\ast}}}}
\Big[
\lVert M(S, \boldsymbol\theta(S)) - M^{\ast}(S, \boldsymbol\theta(S)) \rVert
\Big],
\\
& \leq
\gamma \cdot L \cdot
\mathbb{E}_{\substack{S \sim \rho^{\boldsymbol\theta^{\rm ref}, M^{\ast}} }}
\Big[
\lVert M(S, \boldsymbol\theta(S)) - M^{\ast}(S, \boldsymbol\theta(S)) \rVert
\Big]
+
\lVert \rho^{\boldsymbol\theta, M^{\ast}} - \rho^{\boldsymbol\theta^{\rm ref}, M^{\ast}} \rVert_1
\cdot
\lVert f \rVert_{\infty}.
\end{align*}
Recall Corollary \ref{corollary: upper bound between two distributions},
\begin{equation*}
\lVert \rho^{\boldsymbol\theta, M^{\ast}} - \rho^{\boldsymbol\theta^{\rm ref}, M^{\ast}} \rVert_1
\leq
\frac{\gamma} {1 - \gamma}
\cdot
\mathbb{E}_{S \sim \rho^{\boldsymbol\theta^{\rm ref}, M^{\ast}}}
\Big[
KL(\boldsymbol\theta(S), \boldsymbol\theta^{\rm ref}(S))^{\frac{1} {2}}
\Big]
\leq
\frac{\gamma} {1 - \gamma}
\cdot
\kappa,
\end{equation*}
where $KL(\boldsymbol\theta(S), \boldsymbol\theta^{\rm ref}(S))^{\frac{1} {2}} \leq \kappa$.
Since the state space is uniformly bounded by $B$,
\begin{equation*}
\max_{S} \lVert M(S, \boldsymbol\theta(S)) - M^{\ast}(S, \boldsymbol\theta(S)) \rVert
\leq
\max_{S} \lVert M(S, \boldsymbol\theta(S)) \rVert + \max_{S} \lVert M^{\ast}(S, \boldsymbol\theta(S)) \rVert
\leq
2 B.
\end{equation*}
Therefore,
\begin{equation*}
|V^{\boldsymbol\theta, M} - V^{\boldsymbol\theta, M^{\ast}}|
\leq
\gamma \cdot L \cdot
\mathbb{E}_{\substack{S \sim \rho^{\boldsymbol\theta^{\rm ref}, M^{\ast}} }}
\Big[
\lVert M(S, \boldsymbol\theta(S)) - M^{\ast}(S, \boldsymbol\theta(S)) \rVert
\Big]
+
2B \cdot \kappa \cdot \frac{\gamma} {1 - \gamma}.
\end{equation*}
\end{proof}

\maintheoremNonDecreaseValueFunction*

\begin{proof}
Recall Proposition \ref{proposition: value function upper bound refine},
at the current iteration,
\begin{equation*}
V^{\boldsymbol\theta^{\rm new}, M^{\rm new}}
-
\Big(
\gamma \cdot L \cdot
\mathbb{E}_{\substack{S \sim \rho^{\boldsymbol\theta^{\rm old}, M^{\ast}}}}
\Big[
\lVert M^{\rm new}(S, \boldsymbol\theta(S)) - M^{\ast}(S, \boldsymbol\theta(S)) \rVert
\Big]
+
\frac{2 B \gamma \kappa} {1 - \gamma}
\Big)
\leq
V^{\boldsymbol\theta^{\rm new}, M^{\ast}},
\end{equation*}
since the second term measures the difference between the estimated system and the population retention system $M^{\ast}$,  
$V^{\boldsymbol\theta^{\rm new}, M^{\ast}}$ leads to $0$ difference.

Since $\boldsymbol\theta^{\rm new}$ and $M^{\rm new}$ attain the optimal for the following objective function,
\begin{gather*}
\boldsymbol\theta^{\rm new}, M^{\rm new}
=
\arg \max \limits_{\boldsymbol\theta, M}
V^{\boldsymbol\theta, M}
-
\Big(
\gamma \cdot L \cdot
\mathbb{E}_{\substack{S \sim \rho^{\boldsymbol\theta^{\rm old}, M^{\ast}}}}
\Big[
\lVert M(S, \boldsymbol\theta(S)) - M^{\ast}(S, \boldsymbol\theta(S)) \rVert
\Big]
+
\frac{2 B \gamma \kappa} {1 - \gamma}
\Big)
\\
{\rm s.t.}
\;\;
KL(\boldsymbol\theta_t^{\rm old}(S), \boldsymbol\theta_t(S))^{\frac{1} {2}} \leq \kappa,
\end{gather*}

\begin{align*}
& V^{\boldsymbol\theta^{\rm new}, M^{\rm new}}
-
\Big(
\gamma \cdot L \cdot
\mathbb{E}_{\substack{S \sim \rho^{\boldsymbol\theta^{\rm old}, M^{\ast}}}}
\Big[
\lVert M^{\rm new}(S, \boldsymbol\theta(S)) - M^{\ast}(S, \boldsymbol\theta(S)) \rVert
\Big]
+
\frac{2 B \gamma \kappa} {1 - \gamma}
\Big)
\\
&
\geq
V^{\boldsymbol\theta^{\rm new}, M^{\rm new}}
-
\gamma \cdot L \cdot
\mathbb{E}_{\substack{S \sim \rho^{\boldsymbol\theta^{\rm new}, M^{\ast}}}}
\Big[
\lVert M^{\rm new}(S, \boldsymbol\theta(S)) - M^{\ast}(S, \boldsymbol\theta(S)) \rVert
\Big],
\\
&
\geq
V^{\boldsymbol\theta^{\rm old}, M^{\ast}}
-
\gamma \cdot L \cdot
\mathbb{E}_{\substack{S \sim \rho^{\boldsymbol\theta^{\rm old}, M^{\ast}} }}
\Big[
\lVert M^{\ast}(S, \boldsymbol\theta(S)) - M^{\ast}(S, \boldsymbol\theta(S)) \rVert
\Big],
\\ 
&
=
V^{\boldsymbol\theta^{\rm old}, M^{\ast}},
\end{align*}
the second term equals to $0$ since the oracle dynamical system $M^{\ast}$ is considered.
Therefore,
\begin{equation*}
V^{\boldsymbol\theta^{\rm old}, M^{\ast}}
\leq
V^{\boldsymbol\theta^{\rm new}, M^{\ast}}.
\end{equation*}

\end{proof}

\section{Stability of the Equilibrium State}
\label{sec: proof for the stability of the equilibrium state}
\StableEquilibrium*
\begin{proof}
Recall the surrogate retention system in Eq.~\eqref{eq: surrogate retention system},
for the $\rm i^{th}$ demographic group,
\begin{align*}
\lambda_{t+1}^i 
& =
\beta \circ \kappa^i(\lambda_t^i, \boldsymbol\theta_t)
\cdot
(1 - \lambda_t^i)
+
\sigma \circ \kappa^i(\lambda_t^i, \boldsymbol\theta_t)
\cdot
\lambda_t^i,
\\
& =
\beta \circ \kappa^i(\lambda_t^i, \boldsymbol\theta_t)
+ 
(\sigma \circ \kappa^i(\lambda_t^i, \boldsymbol\theta_t)
-
\beta \circ \kappa^i(\lambda_t^i, \boldsymbol\theta_t))
\cdot
\lambda_t^i.
\end{align*}

We take the derivative of $\lambda_{t+1}^i$ with respect to $\lambda_t^i$,
\begin{equation*}
\frac{\partial \lambda_{t+1}^i} {\partial \lambda_t^i}
=
\frac{\partial \beta \circ \kappa^i(\lambda_t^i, \boldsymbol\theta_t)}
{\partial \lambda_t^i}
+
\Big[
\frac{\partial \sigma \circ \kappa^i(\lambda_t^i, \boldsymbol\theta_t)}
{\partial \lambda_t^i}
-
\frac{\partial \beta \circ \kappa^i(\lambda_t^i, \boldsymbol\theta_t)}
{\partial \lambda_t^i}
\Big]
\cdot
\lambda_t^i
+
\sigma \circ \kappa^i(\lambda_t^i, \boldsymbol\theta_t)
-
\beta \circ \kappa^i(\lambda_t^i, \boldsymbol\theta_t).
\end{equation*}
When $\lambda_t^i = 1$,
suppose that both birth rate and survival rate functions reach their maximum value of $1$,
this can simplify the above expression as follows,
\begin{align*}
\frac{\partial \lambda_{t+1}^i} {\partial \lambda_t^i}
& =
\frac{\partial \sigma \circ \kappa^i(\lambda_t^i, \boldsymbol\theta_t)}
{\partial \lambda_t^i}
+
\sigma \circ \kappa^i(\lambda_t^i, \boldsymbol\theta_t)
-
\beta \circ \kappa^i(\lambda_t^i, \boldsymbol\theta_t),
\\
& =
\frac{\partial \sigma}
{\kappa^i(\lambda_t^i, \boldsymbol\theta_t)}
\cdot
\frac{\kappa^i(\lambda_t^i, \boldsymbol\theta_t)}
{\partial \lambda_t^i}.
\end{align*}
Recall the definition of $ \kappa^i(\lambda_t^i, \boldsymbol\theta_t)$ as defined in Eq.~\eqref{eq: dual of worst-case distribution loss},
\begin{align*}
\kappa^i(\lambda_t^i, \boldsymbol\theta_t)
& =
\inf_{\eta \in \mathbb{R}} \Big(C(\lambda_t^i) \cdot \big( \mathbb{E}_{\mathcal{P}^i} \big[[\Phi(\boldsymbol\theta_t, \mat{x}, y) - \eta]_{+}^2 \big] \big)^{\frac{1} {2}} + \eta \Big), 
\;\;
C(\lambda_t^i) = (2 (1 / \lambda_t^i - 1)^2 + 1)^{\frac{1}{2}},
\\
&
=
C(\lambda_t^i) \cdot \big( \mathbb{E}_{\mathcal{P}^i} \big[[\Phi(\boldsymbol\theta_t, \mat{x}, y) - \eta^{\ast}]_{+}^2 \big] \big)^{\frac{1} {2}} + \eta^{\ast}, 
\;\;
C(\lambda_t^i) = (2 (1 / \lambda_t^i - 1)^2 + 1)^{\frac{1}{2}},
\end{align*}
in which we use $\eta^{\ast}$ as the optimal $\eta$ that leads to the infimum,
in which case, $\eta^{\ast}$ is dependent on $\lambda_t^i$ given $\boldsymbol\theta_t$.
The derivative of $\kappa^i(\lambda_t^i, \boldsymbol\theta_t)$ with respect to $\lambda_t^i$ can be derived as follows,
\begin{align*}
& \frac{\partial \kappa^i(\lambda_t^i, \boldsymbol\theta_t)} {\partial \lambda_t^i}
\\
& =
\frac{\partial C(\lambda_t^i)} {\partial \lambda_t^i} 
\cdot
\Big(
\mathbb{E}_{\mathcal{P}^i} \big[[\Phi(\boldsymbol\theta_t, \mat{x}, y) - \eta^{\ast}]_{+}^2 \big]
\Big)^{\frac{1} {2}}
+
C(\lambda_t^i)
\cdot
\frac{\partial \Big(
\mathbb{E}_{\mathcal{P}^i} \big[[\Phi(\boldsymbol\theta_t, \mat{x}, y) - \eta^{\ast}]_{+}^2 \big]
\Big)^{\frac{1} {2}}}
{\partial \lambda_t^i}
+
\frac{\partial \eta^{\ast}} {\partial \lambda_t^i},
\end{align*}
where 
\begin{equation*}
\frac{\partial C(\lambda_t^i)} {\partial \lambda_t^i}
=
\frac{1} {2} (2 (1 / \lambda_t^i - 1)^2 + 1)^{-\frac{1} {2}}
\cdot
4 (1 / \lambda_t^i - 1)
\cdot
(- (\lambda_t^i)^{-2}),
\end{equation*}
when $\lambda_t^i = 1$,
$\frac{\partial C(\lambda_t^i)} {\partial \lambda_t^i} = 0$,
and $C(\lambda_t^i) = 1$.
Therefore,
\begin{align*}
\frac{\partial \kappa^i(\lambda_t^i, \boldsymbol\theta_t)} {\partial \lambda_t^i}
& =
\frac{\partial \Big(
\mathbb{E}_{\mathcal{P}^i} \big[[\Phi(\boldsymbol\theta_t, \mat{x}, y) - \eta^{\ast}]_{+}^2 \big]
\Big)^{\frac{1} {2}}}
{\partial \lambda_t^i}
+
\frac{\partial \eta^{\ast}} {\partial \lambda_t^i},
\\
&
=
\frac{1} {2} 
\Big(
\mathbb{E}_{\mathcal{P}^i} \big[[\Phi(\boldsymbol\theta_t, \mat{x}, y) - \eta^{\ast}]_{+}^2 \big]
\Big)^{-\frac{1}{2}}
\cdot
\mathbb{E}_{\mathcal{P}^i} \big[2 \cdot [\Phi(\boldsymbol\theta_t, \mat{x}, y) - \eta^{\ast}]_{+} \big]
(- \frac{\partial \eta^{\ast}} {\partial \lambda_t^i})
+
(\frac{\partial \eta^{\ast}} {\partial \lambda_t^i}),
\end{align*}
notice that when $\lambda_t^i = 1$,
$\eta^{\ast}$ approaches $0$ as this leads to population risk in the DRO formulation.
Therefore,
\begin{equation*}
\frac{\partial \kappa^i(\lambda_t^i, \boldsymbol\theta_t)} {\partial \lambda_t^i}
=
\Big(
\mathbb{E}_{\mathcal{P}^i} \big[\Phi(\boldsymbol\theta_t, \mat{x}, y)^2 \big]
\Big)^{-\frac{1}{2}}
\cdot
\mathbb{E}_{\mathcal{P}^i} \big[\Phi(\boldsymbol\theta_t, \mat{x}, y) \big]
\cdot
(- \frac{\partial \eta^{\ast}} {\partial \lambda_t^i})
+
(\frac{\partial \eta^{\ast}} {\partial \lambda_t^i}).
\end{equation*}
Therefore,
\begin{align*}
\frac{\partial \lambda_{t+1}^i} {\partial \lambda_t^i}
& =
\frac{\partial \sigma} {\kappa^i(\lambda_t^i, \boldsymbol\theta_t)}
\cdot
\frac{\partial \kappa^i(\lambda_t^i, \boldsymbol\theta_t)} {\partial \lambda_t^i}
\\
& =
\frac{\partial \sigma}
{\partial \kappa^i(\lambda_t^i, \boldsymbol\theta_t)}
\cdot
\Big(
\Big(
\mathbb{E}_{\mathcal{P}^i} \big[\Phi(\boldsymbol\theta_t, \mat{x}, y)^2 \big]
\Big)^{-\frac{1}{2}}
\cdot
\mathbb{E}_{\mathcal{P}^i} \big[\Phi(\boldsymbol\theta_t, \mat{x}, y) \big]
\cdot
(- \frac{\partial \eta^{\ast}} {\partial \lambda_t^i})
+
(\frac{\partial \eta^{\ast}} {\partial \lambda_t^i})
\Big),
\\
& =
\frac{\partial \sigma} {\kappa^i(\lambda_t^i, \boldsymbol\theta_t)}
\cdot
\frac{\partial \eta^{\ast}} {\partial \lambda_t^i}
\cdot
\Big(
1 -
\frac{\mathbb{E}_{\mathcal{P}^i} \big[\Phi(\boldsymbol\theta_t, \mat{x}, y)\big] }
{\sqrt{
\mathbb{E}_{\mathcal{P}^i} \big[\Phi(\boldsymbol\theta_t, \mat{x}, y)^2 \big]
}}
\Big).
\end{align*}

For a difference equation,
an equilibrium state is stable if the maximum eigenvalue of the Jacobian matrix evaluated at this state is less than $1$.
Since the Jacobian matrix of the surrogate retention system is a diagonal matrix,
its eigenvalues are the diagonal elements.
\begin{equation*}
\max_{i \in [1,2,...,K]} \frac{\partial \lambda_{t+1}^i} {\partial \lambda_t^i} \Big|_{\lambda_t^i = 1}
< 1,
\end{equation*}
which leads to the following condition,
\begin{equation*}
\max_{i \in [1,2,...,K]}
\frac{\partial \sigma} {\kappa^i(\lambda_t^i, \boldsymbol\theta_t)}
\cdot
\frac{\partial \eta^{\ast}} {\partial \lambda_t^i}
\cdot
\Big(
1 -
\frac{\mathbb{E}_{\mathcal{P}^i} \big[\Phi(\boldsymbol\theta_t, \mat{x}, y)\big] }
{\sqrt{
\mathbb{E}_{\mathcal{P}^i} \big[\Phi(\boldsymbol\theta_t, \mat{x}, y)^2 \big]
}}
\Big)
< 1.
\end{equation*}
\end{proof}

\section{Worst-Case Guarantee}
\label{sec: proof for the worst-case guarantee}
\WorstCaseGuarantee*

\begin{proof}
Given a population density $\lambda_t^i$ and a predictive model $\boldsymbol\theta_t$,
the DRO is defined as follows,
\begin{equation*}
\kappa^i(\lambda_t^i, \boldsymbol\theta_t) = \sup_{\mathcal{Q} \in \mathcal{B}(\mathcal{M}^i, r_t^i)} \mathbb{E}_{\mat{z} \sim \mathcal{Q}} \Phi(\boldsymbol\theta_t, \mat{x}, y), 
\;\; 
r_t^i = (1 / \lambda_t^i - 1)^2.
\end{equation*}
Let the population risk be defined as follows,
\begin{equation*}
\hat{\kappa}^i(\lambda_t^i, \boldsymbol\theta_t) = \mathbb{E}_{(\mat{x}, y) \sim \mathcal{P}} \Phi(\boldsymbol\theta_t, \mat{x}, y).
\end{equation*}
Recall the definition of the chi-squared ball around a probability distribution,
\begin{equation*}
\mathcal{B}(\mathcal{P}, r) = \{Q: d_{\mathcal{X}^2}(\mathcal{P}||Q) \leq r\},
\end{equation*}
where $d_{\mathcal{X}^2}(\mathcal{P}||Q) = \int (\frac{d \mathcal{P}} {dQ} - 1)^2 dQ)$ denote the $\mathcal{X}^2$-divergence between two probability distributions $\mathcal{P}$ and $Q$.
Suppose $0 \leq r_1 \leq r_2$,
we have
\begin{equation*}
\mathcal{B}(\mathcal{P}, r_1) \subset \mathcal{B}(\mathcal{P}, r_2).
\end{equation*}
Let
\begin{gather*}
R_1 = \sup_{Q \in \mathcal{B}(\mathcal{P}, r_1)} \mathbb{E}_{\mat{z} \sim Q} \Phi(\boldsymbol\theta_t, \mat{x}, y),
\\
R_2 = \sup_{Q \in \mathcal{B}(\mathcal{P}, r_2)} \mathbb{E}_{\mat{z} \sim Q} \Phi(\boldsymbol\theta_t, \mat{x}, y),
\end{gather*}
we have $R_1 \leq R_2$ since $\mathcal{B}(\mathcal{P}, r_2)$ contains $\mathcal{B}(\mathcal{P}, r_1)$.
When $r_1 = 0$, $R_1 = \mathbb{E}_{(\mat{x}, y) \sim \mathcal{P}^i} \Phi(\boldsymbol\theta_t, \mat{x}, y)$,
the following is true for any $0 \leq r_t^i$,
\begin{equation*}
\mathbb{E}_{(\mat{x}, y) \sim \mathcal{P}^i} \Phi(\boldsymbol\theta_t, \mat{x}, y)
\leq \sup_{Q \in \mathcal{B}(\mathcal{P}^i, r_t^i)} \mathbb{E}_{(\mat{x}, y) \sim Q} \Phi(\boldsymbol\theta_t, \mat{x}, y)
= \kappa^i(\lambda_t^i, \boldsymbol\theta_t).
\end{equation*}

Recall the definition of the retention function in Eq.~\eqref{eq: surrogate retention system},
the loss is non-increasing to both surviving ($\sigma(\cdot)$) and birth ($\beta(\cdot)$) rate functions,
then,
\begin{gather*}
\beta(\kappa^i(\lambda_t^i, \boldsymbol\theta_t))
\leq
\beta(\hat{\kappa}^i(\lambda_t^i, \boldsymbol\theta_t)),
\\
\sigma(\kappa^i(\lambda_t^i, \boldsymbol\theta_t))
\leq
\sigma(\hat{\kappa}^i(\lambda_t^i, \boldsymbol\theta_t)).
\end{gather*}
Given an initial population density $\lambda_0^i$,
for $i \in [1, K]$,
we have 
\begin{equation*}
\lambda_1^i = \beta(\kappa^i(\lambda_0^i, \boldsymbol\theta_0)) (1 - \lambda_0^i)
+
\sigma(\kappa^i(\lambda_0^i, \boldsymbol\theta_0)) \lambda_0^i
\leq
\beta(\hat{\kappa}^i(\lambda_0^i, \boldsymbol\theta_0)) (1 - \lambda_0^i)
+
\sigma(\hat{\kappa}^i(\lambda_0^i, \boldsymbol\theta_0)) \lambda_0^i
=
\hat{\lambda}_1^i.
\end{equation*}
Suppose it is true that
$\lambda_t^i \leq \pi_t^i$.
By induction,
for $i \in [1, K]$,
$t \in [0, T]$,
\begin{equation*}
\lambda_t^i = \beta(\kappa^i(\lambda_t^i, \boldsymbol\theta_t)) (1 - \lambda_t^i)
+
\sigma(\kappa^i(\lambda_t^i, \boldsymbol\theta_t)) \lambda_t^i
\leq
\beta(\hat{\kappa}^i(\lambda_t^i, \boldsymbol\theta_t)) (1 - \lambda_t^i)
+
\sigma(\hat{\kappa}^i(\lambda_t^i, \boldsymbol\theta_t)) \lambda_t^i
=
\hat{\lambda}_t^i.
\end{equation*}
\end{proof}

\vskip 0.2in
\bibliography{sample}

\end{document}